\documentclass{article}

\usepackage{microtype}
\usepackage{graphicx}
\usepackage{subcaption}
\usepackage{booktabs} 
\usepackage{graphicx} 
\usepackage{marvosym}

\usepackage{hyperref}

\usepackage[preprint]{icml2026}

\usepackage{amsmath}
\usepackage{amssymb}
\usepackage{mathtools}
\usepackage{amsthm}
\usepackage[table]{xcolor}
\newcommand{\best}[1]{\textbf{#1}}
\newcommand{\ours}[1]{\cellcolor{green!10}{#1}} 

\usepackage[capitalize,noabbrev]{cleveref}

\theoremstyle{plain}
\newtheorem{theorem}{Theorem}[section]

\newtheorem{lemma}[theorem]{Lemma}

\theoremstyle{definition}

\newtheorem{assumption}[theorem]{Assumption}
\theoremstyle{remark}

\newcommand{\Framework}{LR-DLLM}

\usepackage[textsize=tiny]{todonotes}
\usepackage{threeparttable}
\usepackage[table]{xcolor}
\usepackage{multirow}
\usepackage{listings}
\definecolor{prompt}{RGB}{59, 130, 246}      
\definecolor{iter0}{RGB}{52, 211, 153}       
\definecolor{iter1}{RGB}{251, 191, 36}       
\definecolor{iter2}{RGB}{168, 85, 247}       
\definecolor{iter3}{RGB}{239, 68, 68}        
\usepackage{graphicx}
\usepackage{subcaption} 
\usepackage[T1]{fontenc}
\usepackage{newtxtext}
\usepackage{newtxmath}
\usepackage{algorithm}
\usepackage{algorithmic}
\usepackage{amsmath}
\usepackage{subcaption} 
\usepackage[breakable]{tcolorbox}
\definecolor{prompt}{RGB}{60,120,200} 

\usepackage{pifont}
\newcommand{\cmark}{\ding{51}} 
\newcommand{\xmark}{\ding{55}} 

\icmltitlerunning{Improving Variable-Length Generation in Diffusion Language Models via Length Regularization}

\begin{document}

\twocolumn[
  \icmltitle{
Improving Variable-Length Generation in Diffusion Language Models via Length Regularization}



\icmlsetsymbol{equal}{*}

\begin{icmlauthorlist}
\icmlauthor{Zicong Cheng}{1}
\icmlauthor{Ruixuan Jia}{1}
\icmlauthor{Jia Li}{1}
\icmlauthor{Guo-Wei Yang}{2}
\icmlauthor{Meng-Hao Guo}{1}
\icmlauthor{Shi-Min Hu}{1}
\end{icmlauthorlist}

\icmlaffiliation{1}{Tsinghua University}
\icmlaffiliation{2}{Proxseer Inc.}


\icmlkeywords{Machine Learning, ICML}

\vskip 0.3in
]

\printAffiliationsAndNotice{\ }

\begin{abstract}
Diffusion Large Language Models (DLLMs) are inherently ill-suited for variable-length generation, as their inference is defined on a fixed-length  canvas and implicitly assumes a known target length.
When the length is unknown, as in realistic completion  and infilling, naively comparing confidence across mask lengths becomes systematically biased, leading to under-generation or redundant continuations.
In this paper, we show that this failure arises from an intrinsic length-induced bias in generation confidence estimates, leaving existing DLLMs without a robust way to determine generation length and making variable-length inference unreliable. 
To address this issue, we propose \Framework{}, a length-regularized inference framework for DLLMs that treats generation length as an explicit variable and achieves reliable length determination at inference time.
It decouples semantic compatibility from length-induced uncertainty through an explicit length regularization that corrects biased confidence estimates. 
Based on this, \Framework{} enables dynamic expansion or contraction of the generation span without modifying the underlying DLLM or its training procedure. Experiments show that \Framework{} achieves \textbf{51.3\%} Pass@1 on HumanEval-Infilling under fully unknown lengths (\textbf{+13.4\%} vs. DreamOn) and \textbf{51.5\%} average Pass@1 on four-language McEval (\textbf{+14.3\%} vs. DreamOn).

\end{abstract}

\section{Introduction}


DLLMs are emerging as a promising alternative to autoregressive language models~\cite{DreamCoder,LLADA,LLADA-1.5,LLADA2},
offering a different generation paradigm based on bidirectional denoising over a fixed-length canvas, 
rather than left-to-right factorization.
In particular, DLLMs are well suited for conditional generation scenarios such as fill-in-the-middle (FIM), where both prefix and suffix context are available. 
However, despite these advantages, 
popular DLLMs remain fundamentally designed for fixed-length generation, implicitly assuming that the target span length is known in advance.

This assumption is frequently violated in realistic settings.
In practical completion and infilling tasks,
the length of the missing span is often unknown and variable, and forcing inference to operate on a pre-fixed canvas can lead to systematic failures.
When the chosen length is shorter than required, DLLMs tend to under-generate, producing incomplete or semantically invalid outputs; when the length is overestimated, they exhibit redundant continuations.
Therefore, variable-length generation remains a critical bottleneck that limits the applicability of DLLMs in real-world scenarios


\begin{figure}[t]
    \centering
    \includegraphics[width=\linewidth]{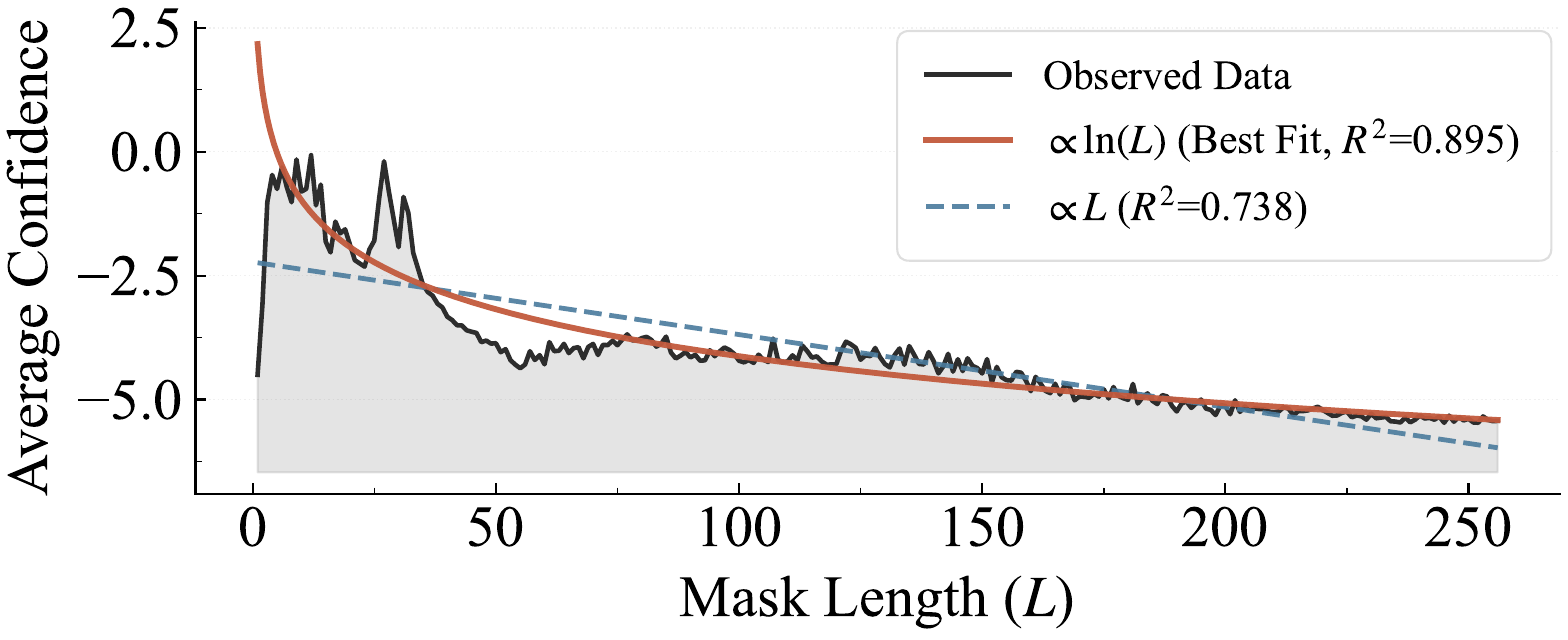}
    \caption{Average confidence as a function of mask length L (log-scaled x-axis). The observed trend is well characterized by a logarithmic dependence ($\propto \log L$), while a linear fit provides a noticeably poorer approximation.}
    \label{fig:motivation}
\end{figure}

Existing attempts at variable-length generation in DLLMs can be broadly divided into training-based and training-free approaches,
yet both suffer from limitations for solving this challenge.
On the one hand, training-based methods typically introduce specialized objectives or auxiliary components,
which require additional retraining and interfere with the model’s original language modeling capability.
On the other hand, training-free approaches are typically more lightweight, but often rely on special control tokens or only allow one-sided adjustments, which restrict flexibility~\cite{DAEDAL}. 
As a result, current solutions still fall short of being simultaneously training-free, model-agnostic, and task-agnostic (Table~\ref{tab:method_comparison}).
In this paper, we focus on developing a training-free variable-length generation that preserves model generality while enabling flexible length adaptation.

\begin{table}[t]
\centering
\caption{Comparison of key properties across variable-length generation methods.
($^{\dagger}$ Task-agnostic here refers to the ability to support both fill-in-the-middle (FIM) and standard left-to-right completion.)
}
\label{tab:method_comparison}
\resizebox{0.48\textwidth}{!}{ 
\begin{tabular}{lcccc}
\toprule
\textbf{Method} & \textbf{Training-free} & \textbf{Model-agnostic} & \textbf{Task-agnostic}$^{\dagger}$ & \textbf{GSM8K} \\
\midrule
DAEDAL          & \cmark & \xmark & \xmark & 63.3 \\
DreamOn         & \xmark & \xmark & \cmark & - \\
FlexMDM        & \xmark & \xmark & \cmark & 67.0 \\
dLLM-Var        & \xmark & \xmark & \xmark &  70.9    \\
\ours{\Framework{}}    & \ours{\cmark} & \ours{\cmark} &  \ours{\cmark} & \ours{\textbf{75.1}} \\
\bottomrule
\end{tabular}
}

\end{table}

We propose \Framework{}, a training-free framework for dynamic variable-length decoding in diffusion LLMs that treats generation length as an explicit inference variable. Rather than fixing the canvas length a priori,
\Framework{} probes candidate lengths and leverages the model’s own generation confidence to search for a more appropriate length. Specifically, it first evaluates a small set of exponentially spaced candidate lengths (e.g., $L\in\{1,2,4,8,\dots\}$) and selects the one maximizing the calibrated confidence as an initial guess. Starting from this initialization, the framework greedily expands or contracts the generation span in a bidirectional manner, guided solely by confidence feedback from the underlying DLLM. The entire procedure operates at inference time and requires no retraining or architectural modification, making it directly applicable to existing diffusion backbones.

A key challenge is comparing confidence across different lengths. Naively selecting the length with the highest average confidence fails due to an intrinsic length-induced bias in DLLMs: confidence drifts approximately logarithmically with canvas length, independent of semantic compatibility. As a result, raw confidence conflates semantic fit with length-dependent uncertainty. We address this by introducing a $\log L$ regularizer that corrects for length bias, yielding a calibrated confidence–length criterion. The regularization coefficient is estimated via lightweight probing, enabling reliable cross-length comparison and robust bidirectional length adjustment during inference.

Experiments show that \Framework{} achieves $51.3\%$ Pass@1 on HumanEval-Infilling with unknown lengths ($+13.4\%$ over DreamOn) and $51.5\%$ average Pass@1 on the four-language McEval benchmark ($+14.3\%$), highlighting its strong generality across models and tasks.

In summary, our key contributions are:
\begin{itemize}
\setlength{\itemsep}{0pt}
\item We identify an intrinsic \emph{length-induced bias} in DLLM confidence estimates, demonstrating that naive confidence comparisons across lengths are unreliable for unknown-length generation.
\item We introduce a length-regularized criterion incorporating a $\log L$ term, providing a principled signal for scoring candidate lengths at inference time.
\item We propose \Framework{}, a training-free inference framework that leverages this criterion via efficient probing and greedy bidirectional adjustment, validating its superiority across diverse benchmarks.
\end{itemize}



\section{Preliminary and Notation}
\subsection{Preliminary}
We study conditional generation problems in which the output must satisfy a hard context constraint,
but the number of tokens to be generated is not specified.
Given an observed context $\mathbf{c}$, the goal is to generate a token span
$\mathbf{y}\in\mathcal{V}^*$ of unknown length and form a completed sequence
$\mathbf{x}=g(\mathbf{c},\mathbf{y})$ that respects the constraint induced by $\mathbf{c}$,
where $g(\cdot,\cdot)$ is a deterministic, task-specific composition function. As a simple example, infilling corresponds to a prefix--suffix context
$\mathbf{c}=(\mathbf{p},\mathbf{s})$ with completion
$\mathbf{x}=\mathbf{p}\Vert\mathbf{y}\Vert\mathbf{s}$.
We denote the (unknown) target span length by $L^\star := |\mathbf{y}|$.

\subsection{Notation}
Table~\ref{tab:notation} summarizes the notation used throughout the paper.

\begin{table}[htbp]
\centering
\small
\caption{Notation used in the paper.}
\setlength{\tabcolsep}{12pt}
\begin{tabular}{@{}ll@{}}
\toprule
Symbol & Description \\
\midrule
$\mathcal{V}$ & Vocabulary \\
$\mathbf{x}=(x_1,\ldots,x_T)$ & Token sequence of length $T$ \\
$|\mathbf{x}|$ & Length of sequence $\mathbf{x}$ \\
$\mathbf{a}\Vert\mathbf{b}$ & Concatenation of two sequences \\
$\langle\mathrm{mask}\rangle$ & Special mask token \\
$\langle\mathrm{mask}\rangle^{L}$ & Mask block of length $L$ \\
$\mathbf{p},\,\mathbf{s}$ & Prefix and suffix context \\
$\mathbf{y}$ & Generated token span \\
$L^\star$ & (Unknown) target span length, $L^\star=|\mathbf{y}|$ \\
\bottomrule
\end{tabular}

\label{tab:notation}
\end{table}

\section{\Framework{}}

\Framework{} is a training-free method for variable-length generation.
Unlike standard DLLM, which assumes a fixed mask length, the target span length under prefix–suffix constraints is unknown.
\Framework{} resolves it by modeling span length as an
\emph{inference-time variable}. 
It relies on a length--confidence signal derived from the model’s uncertainty estimates
to guide length selection.
Using this signal, \Framework{} combines a lightweight probing-based initialization with
inference-time refinement for
the precise length.

\subsection{Overview}
Given a prefix--suffix context $(\mathbf{p},\mathbf{s})$, \Framework{} performs variable-length generation
by explicitly treating the remaining span length as an inference-time decision variable.
Crucially, this requires correcting the intrinsic bias of confidence estimates with respect to length;
without explicit length regularization, confidence scores across different candidate lengths are not
directly comparable, making reliable length selection impossible.

At each step, \Framework{} uses the proposed length--confidence signal to determine
(i) how much content remains to be generated and
(ii) when it is appropriate to commit a new token.
Overall, \Framework{} consists of two stages:
\vspace{-5pt}
\begin{itemize}
\setlength{\itemsep}{0pt}
\item \textbf{Stage I (Initialization by probing).}
A small set of candidate lengths is probed to estimate an instance-specific length regularization
coefficient and to select an initial remaining length that is calibrated to the given context.

\item \textbf{Stage II (Inference-time refinement).}
During generation, the remaining length is locally refined using the length-regularized confidence signal.
A token is committed once the current length becomes locally preferred under this criterion.
This process repeats until the span is fully generated or a global budget is reached.
\end{itemize}

\subsection{Length-Confidence Signal}
\label{sec:conf_length_signal}

To select an appropriate masked span length during inference, we require a scalar signal that allows us
to \emph{compare different candidate lengths under the same prefix--suffix context}.
A natural choice is to measure how confidently the model can fill a masked block of a given length.
In this subsection, we formalize such a confidence measure and show how to correct it for a systematic
length-induced bias.

\paragraph{Span-level confidence from token uncertainty.}
Let $\mathbf{c}=(\mathbf{p},\mathbf{s})$ denote the prefix--suffix context, and let
$q_{\theta}(\cdot \mid \mathbf{x}, \mathbf{c})$ be the token distribution produced by the DLLM
inference procedure for a masked position.
We quantify the confidence at a masked position $i$ using the negative entropy of this distribution:
\begin{equation}
\mathrm{conf}(i; \mathbf{x}, \mathbf{c})
:= \sum_{v\in\mathcal{V}} q_{\theta}\!\left(v \mid \mathbf{x}, \mathbf{c}\right) \quad \log q_{\theta}\!\left(v \mid \mathbf{x}, \mathbf{c}\right).
\end{equation}
Intuitively, a lower-entropy distribution indicates that the model is more certain about the token to be
generated at that position.

Given a length-$L$ masked template
$\mathbf{x}^{(0)}(L)=\mathbf{p}\Vert \langle \mathrm{mask}\rangle^L \Vert \mathbf{s}$,
we aggregate the token-level confidences over the $L$ masked positions and define the span-level confidence
as their average:
\begin{equation}
\mathrm{AVG}_{\mathrm{conf}}(L;\mathbf{p},\mathbf{s})
\;:=\;
\frac{1}{L}\sum_{i\in \mathcal{M}(L)} \mathrm{conf}(i; \mathbf{x}, \mathbf{c}),
\end{equation}
where $\mathcal{M}(L)$ denotes the index set of the masked block.
In practice, all masked token distributions required to compute
$\mathrm{AVG}_{\mathrm{conf}}(L;\mathbf{p},\mathbf{s})$
are obtained from a single standard inference pass of the underlying diffusion model.

\vspace{2pt}
\paragraph{Length-confounded confidence.}
The quantity $\mathrm{AVG}_{\mathrm{conf}}(L;\mathbf{p},\mathbf{s})$ varies with the mask length $L$ for
two distinct reasons.
On the one hand, it reflects whether a span of length $L$ is semantically compatible with the given prefix
and suffix.
On the other hand, it is systematically influenced by the number of masked positions itself, because
changing $L$ alters the uncertainty structure of masked prediction even under an identical context.
To disentangle these effects, we model the observed span-level confidence as
\begin{equation}
\mathrm{AVG}_{\mathrm{conf}}(L;\mathbf{p},\mathbf{s})
\;=\;
S(L;\mathbf{p},\mathbf{s}) \;+\; R(L;\mathbf{p},\mathbf{s}),
\label{eq:conf_decomp}
\end{equation}
where $S(L;\mathbf{p},\mathbf{s})$ captures the context--length compatibility that we seek to evaluate,
and $R(L;\mathbf{p},\mathbf{s})$ represents a systematic shift induced by length.

\vspace{2pt}
\paragraph{Length--confidence signal.}
While $S(L;\mathbf{p},\mathbf{s})$ is not directly observable, the aggregate confidence
$\mathrm{AVG}_{\mathrm{conf}}(L;\mathbf{p},\mathbf{s})$ is.
We therefore estimate and subtract the length-induced bias term $R(L;\mathbf{p},\mathbf{s})$.
As shown in Appendix~\ref{app:logL_bias}, under mild assumptions this bias admits a logarithmic form
$R(L;\mathbf{p},\mathbf{s}) \;=\; k \log L$,
where the coefficient $k$ is instance-dependent and may vary across contexts.
This yields our length--confidence signal:
\begin{equation}
\mathrm{CL}(L;\mathbf{p},\mathbf{s})
\;:=\;
\mathrm{AVG}_{\mathrm{conf}}(L;\mathbf{p},\mathbf{s}) \;-\; k \log L,
\label{eq:cl_signal}
\end{equation}
which removes the systematic length effect from the raw confidence score.
The resulting signal serves as a practical proxy for the semantic compatibility
$S(L;\mathbf{p},\mathbf{s})$ and enables direct comparison between candidate lengths.
In the following subsections, we describe how \Framework{} estimates the coefficient $k$ and selects an
initial length (Stage~I), and how it performs greedy bidirectional length adjustment during inference
(Stage~II).

\subsection{Stage~I: Parameter Estimation and Length Selection}
\label{sec:stage1}
Our length--confidence signal is
$\mathrm{CL}(L;\mathbf{p},\mathbf{s})=\mathrm{AVG}_{\mathrm{conf}}(L;\mathbf{p},\mathbf{s})-k\log L$.
In deployment, the coefficient $k$ is previously unknown and may vary between instances.
Stage~I therefore performs two tasks: First, it estimates an instance-specific $\hat{k}$ with a small number of probes,
and then selects an initialization length $L_{\mathrm{rem}}$ that is already adapted to the given context.
The full procedure is provided in Algorithm~\ref{alg:stage1}.

\paragraph{Exponential probing of candidate lengths.}
Given a maximum probing length $\mathrm{MAX\_LENGTH}$, we probe a compact set of exponentially spaced lengths
$\mathcal{L}_{\exp}=\{1,2,4,\dots,2^m\}$ where $2^m\le \mathrm{MAX\_LENGTH}$.
For  each $L\in\mathcal{L}_{\exp}$, we form the masked template
$\mathbf{x}^{(0)}(L)=\mathbf{p}\Vert \langle \mathrm{mask}\rangle^{L}\Vert \mathbf{s}$
and run a single inference pass to obtain the average confidence
We denote $A(L) := \mathrm{AVG}_{\mathrm{conf}}(L;\mathbf{p},\mathbf{s})$ for brevity.
This exponential schedule keeps the probing cost logarithmic in $\mathrm{MAX\_LENGTH}$ while still covering a wide range of lengths.

\paragraph{Estimating $\hat{k}$ from the long-span regime.}
To estimate the length coefficient, we fit a linear model using the probed generation lengths.
\begin{equation}
A(L)\approx \alpha + k\log L,
\end{equation}
where the fitted slope yields the instance-specific $\hat{k}$.
The intuition is that when the masked span is sufficiently long, the context-dependent component of $A(L)$ changes slowly with $L$,
so the remaining systematic variation is dominated by the length bias $\log L$.

\paragraph{Selecting an initial length.}
With $\hat{k}$, we compute $\mathrm{CL}(L)=A(L)-\hat{k}\log L$ for all $L\in\mathcal{L}_{\exp}$ and choose
\begin{equation}
L_{\mathrm{rem}} \;=\; \arg\max_{L\in\mathcal{L}_{\exp}} \mathrm{CL}(L).
\end{equation}
This produces a strong, context-adaptive initialization, which significantly reduces the refinement cost in Stage~II and improves robustness.

\begin{algorithm}[t]
\caption{Stage I: Estimate $\hat{k}$ and Choose Initial Length via Length-Confidence Signal}
\label{alg:stage1}
\small
\begin{algorithmic}[1]
\REQUIRE Context $(\mathbf{p},\mathbf{s})$, maximum probing length $\text{MAX\_LENGTH}$
\ENSURE Estimated slope $\hat{k}$ and initial remaining length $L_{\mathrm{rem}}$
\STATE $\mathcal{L}_{\mathrm{exp}} \leftarrow \{1,2,4,\ldots,2^m\}$ where $2^m \le \text{MAX\_LENGTH}$
\FOR{$L \in \mathcal{L}_{\mathrm{exp}}$}
    \STATE $\mathbf{x}^{(0)}(L) \leftarrow \mathbf{p}\Vert \langle\mathrm{mask}\rangle^{L}\Vert \mathbf{s}$
    \STATE $A(L) \leftarrow \mathrm{AVG}_{\mathrm{conf}}\!\left(L;\mathbf{p},\mathbf{s}\right)$ \hfill // one inference pass
\ENDFOR
\STATE Fit $A(L) \approx \alpha + k \log L$ and obtain $\hat{k}$
\FOR{$L \in \mathcal{L}_{\mathrm{exp}}$}
    \STATE $\mathrm{CL}(L) \leftarrow A(L) - \hat{k}\log L$
\ENDFOR
\STATE $L_{\mathrm{rem}} \leftarrow \arg\max_{L\in\mathcal{L}_{\mathrm{exp}}}\mathrm{CL}(L)$
\STATE \textbf{return} $\hat{k}, L_{\mathrm{rem}}$
\end{algorithmic}
\end{algorithm}

\subsection{Stage~II: Adjust during Inference}
\label{sec:stage2}

Stage~I provides an instance-specific estimate $\hat{k}$ and an initial remaining length $L_{\mathrm{rem}}$.
Stage~II then performs \emph{inference-time} refinement by treating the remaining length $L_{\mathrm{rem}}$ as a
first-class variable that can be updated as the prefix grows.
At a high level, Stage~II alternates between two operations: (i) adjust $L_{\mathrm{rem}}$ by a small amount if doing so
improves the length--confidence signal under the current context, and (ii) commit one token once the current length is
locally preferred. The complete procedure is given in Algorithm~\ref{alg:stage2}.

\paragraph{Local length refinement.}
Given the current prefix--suffix context $(\mathbf{p},\mathbf{s})$ and remaining length $L_{\mathrm{rem}}$, we consider a
small neighborhood of candidate lengths around the current value (i.e., one-step expansion or contraction, with boundary
handling at $L_{\mathrm{rem}}{=}1$).
We evaluate the length--confidence signal $\mathrm{CL}(\ell;\mathbf{p},\mathbf{s})$ on these candidates and greedily move
$L_{\mathrm{rem}}$ to the best-scoring neighbor if it improves the score.
This local refinement is repeated until no neighboring length yields a higher signal, at which point the current
$L_{\mathrm{rem}}$ is treated as locally stable for the present context.

\begin{algorithm}[t]
\caption{Stage II: Greedy Variable-Length Inference with Local Length Search}
\label{alg:stage2}
\small
\begin{algorithmic}[1]
\REQUIRE Context $(\mathbf{p},\mathbf{s})$, initial remaining length $L_{\mathrm{rem}}$, budget $\text{MAX\_GEN}$,
signal $\mathrm{CL}(L;\mathbf{p},\mathbf{s})=\mathrm{AVG}_{\mathrm{conf}}(L;\mathbf{p},\mathbf{s})-\hat{k}\log L$
\ENSURE Generated span $\hat{\mathbf{y}}$ and completed sequence $\hat{\mathbf{x}}$
\STATE $\hat{\mathbf{y}}\leftarrow \varnothing$, \quad $n_{\mathrm{gen}}\leftarrow 0$, \quad $t\leftarrow 0$
\WHILE{$L_{\mathrm{rem}} > 0$ \AND $n_{\mathrm{gen}} + L_{\mathrm{rem}} \le \text{MAX\_GEN}$}
    \STATE \textit{// (1) Local search: adjust remaining length}
    \REPEAT
        \STATE $\mathcal{N} \leftarrow \{\max(1,L_{\mathrm{rem}}-1),\,L_{\mathrm{rem}},\,L_{\mathrm{rem}}+1\}$
        \STATE $L_{\text{best}}
 \leftarrow \arg\max_{\ell\in\mathcal{N}} \mathrm{CL}(\ell;\mathbf{p},\mathbf{s})$
        \IF{$L_{\text{best}}
 \neq L_{\mathrm{rem}}$ \AND $n_{\mathrm{gen}} + L_{\text{best}}
 \le \text{MAX\_GEN}$}
            \STATE $t\leftarrow t+1$; $L_{\mathrm{rem}} \leftarrow L_{\text{best}}
$
        \ELSE
            \STATE $t\leftarrow t+1$; $L_{\mathrm{rem}} \leftarrow \min\!\left(L_{\text{best}}
,\; \text{MAX\_GEN} - n_{\mathrm{gen}}\right)$ 
        \ENDIF
    \UNTIL{$L_{\mathrm{rem}}$ unchanged}

    \STATE \textit{// (2) Commit one token at the left boundary}
    \STATE $\mathbf{x}^{(0)} \leftarrow \mathbf{p}\Vert \langle\mathrm{mask}\rangle^{L_{\mathrm{rem}}}\Vert \mathbf{s}$
    \STATE Obtain token distributions for the masked positions by one inference step on $\mathbf{x}^{(0)}$
    \STATE $\hat{y}_t \leftarrow \mathrm{Sample}\!\left(P(x_i \mid \mathbf{x}^{(0)}) \quad \text{for the leftmost masked position } i\right)$

    \STATE $\hat{\mathbf{y}}\leftarrow \hat{\mathbf{y}}\Vert \hat{y}_t$; \quad $\mathbf{p}\leftarrow \mathbf{p}\Vert \hat{y}_t$
    \STATE $n_{\mathrm{gen}}\leftarrow n_{\mathrm{gen}}+1$; \quad $L_{\mathrm{rem}}\leftarrow L_{\mathrm{rem}}-1$
\ENDWHILE
\STATE $\hat{\mathbf{x}}\leftarrow \mathbf{p}\Vert\hat{\mathbf{y}}\Vert\mathbf{s}$
\STATE \textbf{return} $\hat{\mathbf{y}}, \hat{\mathbf{x}}$
\end{algorithmic}
\end{algorithm}

\paragraph{Token commitment and problem reduction.}
Once the remaining length is locally stable, we perform one inference pass on the template
$\mathbf{p}\Vert\langle\mathrm{mask}\rangle^{L_{\mathrm{rem}}}\Vert\mathbf{s}$ to obtain token distributions over the
masked block, then materialize a single token at the left boundary of the masked region (i.e., the masked position
adjacent to the prefix).
We append this token to the prefix, record it in the generated span, and decrement the remaining length by one:
\begin{equation}
\mathbf{p} \leftarrow \mathbf{p}\Vert \hat{y}_t, \qquad
L_{\mathrm{rem}} \leftarrow L_{\mathrm{rem}}-1.
\end{equation}
This reduces the original infilling instance to a smaller subproblem with the same suffix $\mathbf{s}$ and an updated
prefix, after which we re-enter the local refinement step.

\paragraph{Termination.}
Stage~II always terminates; we provide a formal justification and complexity analysis in
Appendix~\ref{app:appendix_complexity}.

\section{Experiment}

In this section, we conduct extensive experiments to demonstrate that \Framework{} exhibits strong generalization across models and broad applicability across tasks. Our empirical study is organized around the following research questions: \vspace{-10pt}

\begin{itemize} 
\setlength{\itemsep}{0pt} 
\item \textbf{RQ1:} How does \Framework{} perform on infilling tasks? 
\item \textbf{RQ2:} How well does \Framework{} generalize across different programming languages for infilling? 
\item \textbf{RQ3:} How well does \Framework{} generalize to general-purpose benchmarks beyond infilling?
\item \textbf{RQ4:} What is the impact of the $\log L$ regularization term on the performance of \Framework{}?
\item \textbf{RQ5:} What is the contribution of Stage~II to the overall effectiveness of \Framework{}?
\end{itemize}

\subsection{Experimental Settings}

\textbf{Datasets.}
We evaluate \Framework{} on both code infilling and general-purpose reasoning tasks.
For infilling, we use HumanEval-Infilling~\cite{humaneval_infilling}.
To assess multilingual infill robustness, we additionally adopt McEval~\cite{Mceval}, covering four programming languages: C++, Python, JavaScript, and Rust.
For general-purpose evaluation, we consider HumanEval~\cite{humaneval} for code generation, GSM8K~\cite{GSM8K} and MATH500~\cite{MATH500} for  math reasoning.

\textbf{Models.}
For code infilling, we test representative diffusion-based code models, including LLaDA-8B~\cite{LLADA}, LLaDA-1.5~\cite{LLADA-1.5}, LLaDA-MoE-7B~\cite{LLADA-MOE}, Dream~\cite{Dream}, and DreamCoder~\cite{DreamCoder}.
For general-purpose benchmarks, we evaluate the subset of these models that are trained or tuned for broader instruction-following and reasoning capabilities.

\textbf{Baselines.}
On fill-in-the-middle (FIM) infilling, we compare against two training-based variable-length approaches, DreamOn~\cite{Dreamon2025} and FlexMDM~\cite{FlexMDM}.
For general-purpose tasks, we include DAEDAL~\cite{DAEDAL} as a representative training-free baseline.

\textbf{Metrics.}
For coding-related tasks (including infilling and code generation), we report Pass@1.
For GSM8K and MATH500, we use accuracy as  evaluation metric.

\subsection{Main Results on Code Infilling (RQ1)}
\begin{table*}[t]
\centering
\small
\caption{HumanEval-Infilling (Pass@1). Within each backbone, \textbf{bold} denotes the best score in the corresponding column. \cellcolor{green!10}{Shaded} cells indicate results obtained by \Framework{} on the same backbone. Mean is averaged over Random Span, Single-line, and Multi-line. For DreamOn, we set the initial mask length to $1$ to isolate its variable-length adjustment capability.
}
\setlength{\tabcolsep}{15pt}
\begin{tabular}{llcccc}
\toprule
\multirow{2}{*}{\textbf{Model}} &
\multirow{2}{*}{\textbf{Method}} &
\multicolumn{4}{c}{\textbf{HumanEval-Infilling (Pass@1)}} \\
\cmidrule(lr){3-6}
 &  & \textbf{Random Span} & \textbf{Single-line} & \textbf{Multi-line} & \textbf{Mean} \\
\midrule

\multirow{3}{*}{\textbf{LLaDA-8B}} &
Baseline  & 16.0 & 48.3 & 21.1 & 23.4 \\
& +FlexMDM & -- & 65.0 & --& -- \\
& +\Framework{}    & \ours{\best{19.1}} & \ours{\best{69.4}} & \ours{\best{35.9}} & \ours{\best{36.7}} \\
\addlinespace[4pt]
\cmidrule(lr){1-6}

\multirow{2}{*}{\textbf{LLaDA-1.5}} &
Baseline & 16.5 & 48.8 & 32.2 & 31.2 \\
& +\textbf{\Framework{}}  & \ours{\best{20.5}} & \ours{\best{68.9}} & \ours{\best{34.6}} & \ours{\best{36.1}} \\
\addlinespace[4pt]
\cmidrule(lr){1-6}

\multirow{2}{*}{\textbf{LLaDA-MoE}} &
Baseline & 14.6 & 48.8 & 38.4 & 35.1 \\
& +\textbf{\Framework{}}  & \ours{\best{21.5}} & \ours{\best{71.3}} & \ours{\best{41.4}} & \ours{\best{40.8}} \\
\addlinespace[4pt]
\cmidrule(lr){1-6}

\multirow{2}{*}{\textbf{Dream-7B}} &
Baseline & 8.17 & 48.2 & 21.9 & 22.5 \\
& +\textbf{\Framework{}}  & \ours{\best{19.6}} & \ours{\best{76.7}} & \ours{\best{46.3}} & \ours{\best{44.8}} \\
\addlinespace[4pt]
\cmidrule(lr){1-6}

\multirow{3}{*}{\textbf{DreamCoder-7B}} &
Baseline  & 14.3 & 55.5 & 43.2 & 39.1 \\
& +DreamOn & \best{21.6} & 73.5 & 36.1 & 37.9 \\
& +\textbf{\Framework{}}    & \ours{\best{23.2}} & \ours{\best{81.6}} & \ours{\best{53.8}} & \ours{\best{51.3}} \\
\bottomrule
\end{tabular}
\label{tab:humaneval_infilling}
\end{table*}
Table~\ref{tab:humaneval_infilling} presents the main results on HumanEval-Infilling.
Unless otherwise specified, the \textit{Baseline} infills with a fixed length of 64 \texttt{[MASK]} tokens, whereas all variable-length methods (including \Framework{}, DreamOn, and FlexMDM) are evaluated under the same adjustment budget with \texttt{MAX\_LENGTH}=128, allowing automatic expansion and contraction during inference.
Under this controlled comparison, \Framework{} yields consistent improvements across every backbone, indicating strong model-agnostic generalization.

On average, \Framework{} substantially increases mean Pass@1 across models, despite their very different starting points.
For LLaDA-8B, the mean rises from 23.4 to 36.7, corresponding to a 56.8\% relative improvement; for LLaDA-1.5, it improves from 31.2 to 36.1, a 15.7\% gain.
LLaDA-MoE increases from 35.1 to 40.8, a 16.2\% gain, while the Dream-family models benefit even more: Dream-7B jumps from 22.5 to 44.8, nearly doubling performance with a 99.1\% relative improvement, and DreamCoder-7B increases from 39.1 to 51.3, a 31.2\% gain.
The fact that \Framework{} improves both weaker and stronger backbones suggests that the proposed inference mechanism provides broadly useful inductive bias rather than exploiting model-specific quirks.

The gains concentrate on the most length-sensitive regime, namely Multi-line infilling, where fixed masking frequently leads to under-generation or redundant continuations.
Across models, \Framework{} yields the largest improvements in this setting, with consistent gains also observed for Single-line infilling.
These trends directly validate the core design of \Framework{}: by treating span length as an explicit inference variable and enabling bidirectional length adjustment, the method mitigates both semantic gaps and redundancy when the target length is unknown.

\Framework{} further demonstrates improved reliability over existing variable-length baselines.
On DreamCoder-7B, the training-based DreamOn exhibits inconsistent behavior across span regimes, improving some settings while degrading Multi-line infilling.
In contrast, \Framework{} consistently improves Random Span, Single-line, and Multi-line infilling simultaneously, resulting in the strongest overall performance.
This uniform gain across heterogeneous span lengths indicates more robust length selection, achieved by explicitly correcting length-induced confidence bias rather than relying on heuristic or initialization-dependent length priors.

Overall, these results show that \Framework{} consistently delivers large and robust gains, especially on long and challenging infilling instances. 
Importantly, these improvements are achieved with a modest inference-time overhead: while \Framework{} performs local probing and bidirectional adjustment, the average number of additional forward calls per generated token remains small and bounded across benchmarks (Table~\ref{tab:fc_ratio_stats}). 
This indicates that \Framework{} effectively optimizes the generation span with an approximately constant average cost, rather than incurring unbounded or length-proportional overhead.

\subsection{Cross-Language Generalization for Infilling (RQ2)}
\begin{table*}[t]
\centering
\small
\caption{\textbf{Multilingual code infilling results.} Pass@1 (\%) on Single-line and Multi-line infilling across four programming languages (JavaScript, Python, C++, and Rust). Within each backbone, \textbf{bold} denotes the best score in the corresponding column. \cellcolor{green!10}{Shaded} cells indicate results obtained by applying \Framework{} to the same backbone. \textbf{Mean} is the sample-size-weighted average over the eight language--setting columns. For DreamOn, we set the initial mask length to $1$.}
\setlength{\tabcolsep}{5.2pt}
\begin{tabular}{llccccccccc}
\toprule
\multirow{2}{*}{\textbf{Model}} &
\multirow{2}{*}{\textbf{Method}} &
\multicolumn{2}{c}{\textbf{JavaScript}} &
\multicolumn{2}{c}{\textbf{Python}} &
\multicolumn{2}{c}{\textbf{C++}} &
\multicolumn{2}{c}{\textbf{Rust}} &
\multirow{2}{*}{\textbf{Mean}} \\
\cmidrule(lr){3-4}\cmidrule(lr){5-6}\cmidrule(lr){7-8}\cmidrule(lr){9-10}
 & & \textbf{Single} & \textbf{Multi} & \textbf{Single} & \textbf{Multi} & \textbf{Single} & \textbf{Multi} & \textbf{Single} & \textbf{Multi} & \\
\midrule

\rowcolor{gray!12}
\multicolumn{11}{c}{\textbf{Diffusion Models}} \\

LLaDA-8B     & Baseline             & 62.5 & 16.3 & 68.9 & 10.0 & 57.3 & 23.1 & 51.1 & 7.61 & 37.0 \\
LLaDA-8B     & +\textbf{\Framework{}} &
\ours{\best{68.8}} & \ours{\best{27.5}} & 
\ours{\best{70.0}} & \ours{\best{22.0}} &
\ours{\best{68.3}} & \ours{\best{25.6}} &
\ours{\best{60.9}} & \ours{\best{12.0}} &
\ours{\best{44.2}} \\
\addlinespace[2pt]
\cmidrule(lr){1-11}

LLaDA-MoE    & Baseline             & 62.5 & 12.5 & 84.4 & 11.1 & 61.0 & 31.7 & 57.6 & 7.61 & 41.0 \\
LLaDA-MoE    & +\textbf{\Framework{}} &
\ours{\best{76.3}} & \ours{\best{28.8}} &
\ours{\best{85.6}} & \ours{\best{26.7}} &
\ours{\best{75.6}} & \ours{\best{32.9}} &
\ours{\best{64.1}} & \ours{\best{25.0}} &
\ours{\best{51.8}} \\
\addlinespace[2pt]
\cmidrule(lr){1-11}

Dream-7B     & Baseline             & 53.8 & 17.5 & 66.7 & 8.88 & 61.0 & 13.4 & 39.1 & 8.70 & 35.8 \\
Dream-7B     & +\textbf{\Framework{}} &
\ours{\best{75.0}} & \ours{\best{33.8}} &
\ours{\best{72.2}} & \ours{\best{23.3}} &
\ours{\best{82.9}} & \ours{\best{24.4}} &
\ours{\best{59.8}} & \ours{\best{14.1}} &
\ours{\best{47.9}} \\
\addlinespace[2pt]
\cmidrule(lr){1-11}

DreamCoder-7B& Baseline             & 66.3 & 12.5 & 73.3 & 10.0 & 57.3 & 20.7 & 52.2 & 6.52 & 37.2 \\
DreamCoder-7B& +DreamOn             &
62.5 & 18.8 &
\best{82.2} & 23.3 &
39.1 & 15.9 &
27.2 & 7.60 & 37.2 \\
DreamCoder-7B& +\textbf{\Framework{}} &
\ours{\best{78.8}} & \ours{\best{38.8}} &
\ours{77.8} & \ours{\best{30.0}} &
\ours{\best{76.8}} & \ours{\best{28.0}} &
\ours{\best{62.0}} & \ours{\best{21.7}} &
\ours{\best{51.5}} \\
\addlinespace[4pt]

\rowcolor{gray!12}
\multicolumn{11}{c}{\textbf{Autoregressive (AR) Models}} \\
DeepSeek-Coder-6.7B & Baseline & 64.0 & 40.0 & 66.0 & 32.4 & 60.0 & 33.8 & 54.7 & 33.2 & 47.9 \\

\bottomrule
\end{tabular}
\label{tab:mceval_infilling}
\end{table*}
Table~\ref{tab:mceval_infilling} evaluates cross-language infilling on McEval.
Unless otherwise specified, the \textit{Baseline} fills a fixed span of 32 \texttt{[MASK]} tokens, while all variable-length methods (including \Framework{} and DreamOn) use the same adjustment budget with \texttt{MAX\_LENGTH}=256, enabling automatic expansion and contraction at inference time.
Within this unified setting, \Framework{} consistently improves performance across all four languages and for both Single-line and Multi-line infilling, demonstrating strong cross-language generalization without language-specific tuning.

Averaged over languages and settings, \Framework{} delivers sizable relative gains on every diffusion backbone.
The mean score of LLaDA-8B increases from 37.0 to 44.2, corresponding to a relative improvement of 19.5\%;
LLaDA-MoE rises from 41.0 to 51.8, a 26.3\% gain;
Dream-7B improves from 35.8 to 47.9, yielding a 33.8\% gain;
and DreamCoder-7B jumps from 37.2 to 51.5, which translates to a 38.4\% gain.
Although the absolute baseline performance varies substantially across models, \Framework{} consistently provides strong improvements, indicating that our length-aware inference mechanism is not tied to a particular backbone or training recipe.

The gains are especially pronounced on the Multi-line setting, where fixed masking most frequently causes length mismatch.
For instance, on LLaDA-MoE, Multi-line accuracy more than doubles across multiple languages:
JavaScript increases from 12.5 to 28.8, improving by 130.4\%;
Python rises from 11.1 to 26.7, improving by 140.5\%;
and Rust climbs from 7.61 to 25.0, improving by 228.5\%.
Dream-7B shows the same trend, with Multi-line improving by 93.1\% on JavaScript (17.5 to 33.8) and by 162.4\% on Python (8.88 to 23.3).
These consistent, large relative gains on Multi-line support our central claim: explicitly optimizing span length via bidirectional adjustment is critical for robust infilling when the missing span length is unknown.

Compared with the training-based DreamOn baseline, \Framework{} is also markedly more robust under cross-language shift.
On DreamCoder-7B, DreamOn yields uneven transfer: while it improves Python, it significantly degrades Single-line performance on C++ from 57.3 to 39.1 and on Rust from 52.2 to 27.2, leaving the overall mean unchanged at 37.2.
In contrast, \Framework{} improves DreamCoder-7B consistently across all languages and settings, including large Multi-line gains such as JavaScript from 12.5 to 38.8 (a 210.4\% increase) and Rust from 6.52 to 21.7 (a 232.8\% increase).
This suggests that \Framework{} provides a more stable, model-agnostic way to handle length uncertainty than methods that rely on training-time heuristics.

\subsection{Results on General-Purpose Benchmarks (RQ3)}
\begin{table}[t]
\centering
\small
\setlength{\tabcolsep}{4.5pt}
\renewcommand{\arraystretch}{1.15}
\caption{We compare fixed-length generation with different mask lengths ($L\in\{32,64,128\}$), DAEDAL with initial $L{=}1$, and \Framework{} under the same backbones. ``--'' denotes results not available (the method is not applicable or not reported).}
\label{tab:general_purpose_rq3}
\begin{tabular}{lccccc}
\toprule
\multirow{2}{*}{\textbf{Benchmark}} &
\multicolumn{3}{c}{\textbf{Fixed-Length}} &
\multirow{2}{*}{\textbf{DAEDAL}} & \multirow{2}{*}{\textbf{\Framework{}}} \\
\cmidrule(lr){2-4}
& \textbf{32} & \textbf{64} & \textbf{128}  \\
\midrule

\rowcolor{gray!20}
\multicolumn{6}{c}{\textit{LLADA-8B}} \\

\textbf{GSM8K}
 & 40.2 & 48.0 & 67.9 & 63.3  & \ours{\textbf{75.1}} \\

\textbf{HumanEval}
 & 12.2 & 18.9 & 26.2 & 27.4  & \ours{\textbf{31.7}} \\

\textbf{MATH500}
 & 22.6 & 22.6 & 23.6 & 26.4 & \ours{\textbf{39.8}} \\

\rowcolor{gray!20}
\multicolumn{6}{c}{\textit{LLADA-1.5}} \\

\textbf{GSM8K}
 & 41.2 & 49.4 & 71.0 & 66.5 & \ours{\textbf{76.7}} \\

\textbf{HumanEval}
 & 16.5 & 18.3 & 22.0 & 32.3 & \ours{\textbf{34.8}} \\

\textbf{MATH500}
 & 24.0 & 22.0 & 23.4 & 24.0 & \ours{\textbf{38.6}} \\

\rowcolor{gray!20}
\multicolumn{6}{c}{\textit{Dream-7B}} \\

\textbf{GSM8K}
 & 24.0 & 52.7 & 75.9 & --  & \ours{\textbf{76.6}} \\

\textbf{HumanEval}
 & 12.8 & 25.9 & 48.3 & --  & \ours{\textbf{53.6}} \\

\textbf{MATH500}
 & 14.2 & 24.2 & 38.6  & -- & \ours{\textbf{41.4}} \\

\bottomrule
\end{tabular}
\end{table}
Table~\ref{tab:general_purpose_rq3} reports results on GSM8K, HumanEval, and MATH500.
We compare fixed-length generation with different mask lengths $L\in\{32,64,128\}$, DAEDAL initialized with $L{=}1$, and \Framework{} under the same diffusion backbones.
Overall, \Framework{} consistently improves over fixed-length decoding across tasks and models, and it also surpasses DAEDAL on the LLADA backbones where DAEDAL is applicable.

The most significant gains appear on MATH500.
On LLaDA-8B, \Framework{} raises MATH500 from 26.4 to 39.8, a 50.8\% relative improvement over DAEDAL, and it also improves GSM8K from 63.3 to 75.1 by 18.6\%.
On LLaDA-1.5, \Framework{} boosts MATH500 from 24.0 to 38.6, achieving a 60.8\% relative improvement, while increasing GSM8K from 66.5 to 76.7 by 15.3\%.
These results suggest that explicit length optimization is especially beneficial for math-intensive reasoning, where suboptimal length often manifests as incomplete derivations or redundant steps.

DAEDAL cannot be applied to non-LLADA diffusion backbones, so no DAEDAL results are reported for Dream-7B.
Nevertheless, \Framework{} still delivers striking improvements over the best fixed-length setting on Dream-7B, most notably on HumanEval, where performance jumps from 48.3 at $L{=}128$ to 53.6, a 11.0\% relative gain, and on MATH500, where it increases from 38.6 to 41.4, a 7.3\% relative gain.
Meanwhile, GSM8K remains essentially unchanged, with \Framework{} achieving 76.6 compared with the best fixed-length result of 75.9, indicating that \Framework{} can preserve strong performance on less length-sensitive benchmarks while substantially improving the most length-critical ones.
\begin{figure}[t]
    \centering
    \includegraphics[width=\linewidth]{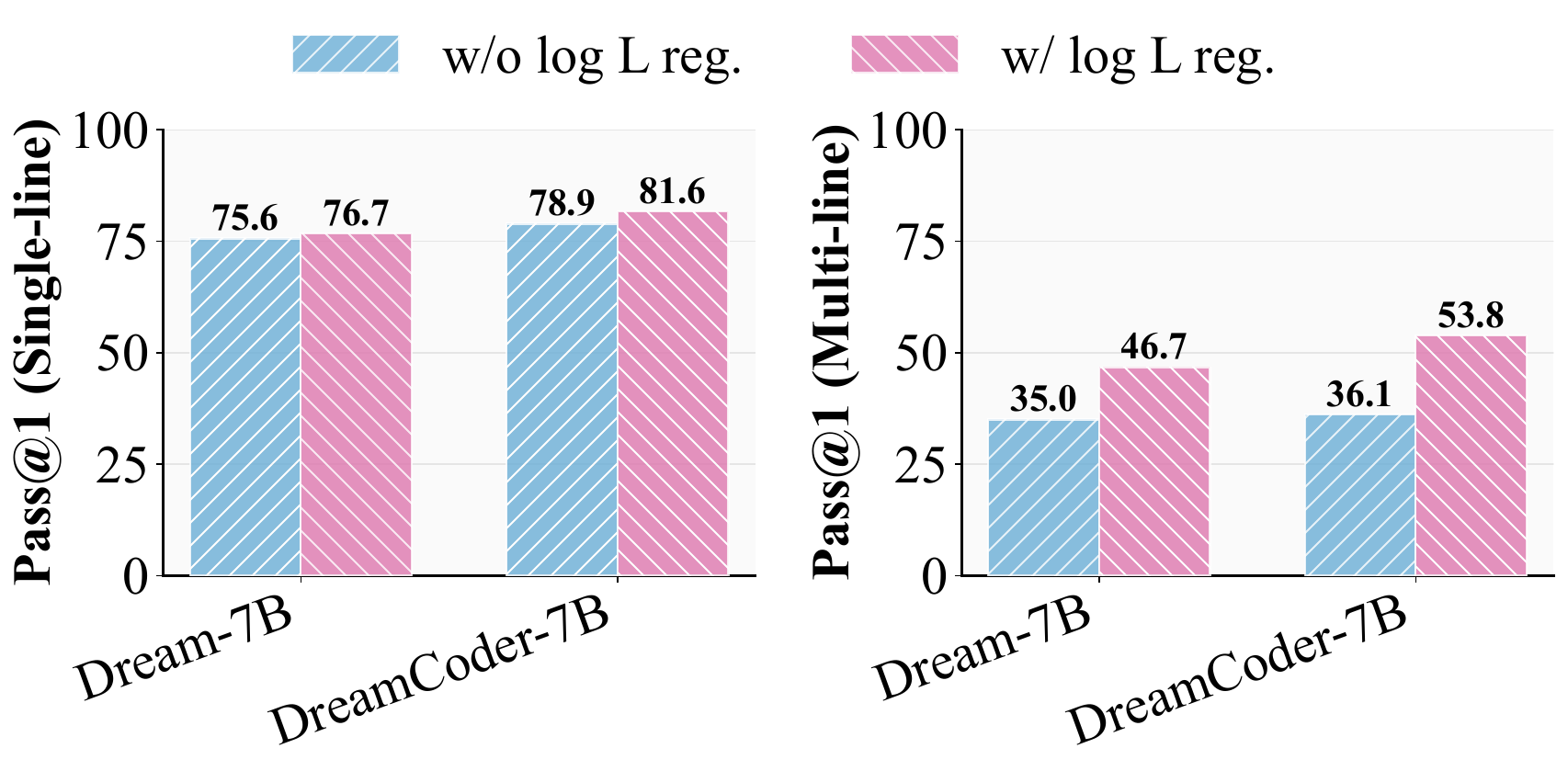}
    \caption{Ablation of the $\log L$ regularization term in \Framework{} on HumanEval-Infilling (Pass@1). All methods use the same length adjustment budget with \texttt{MAX\_LENGTH}=128.}
    \label{fig:ablation_logL_bar}
\end{figure}
\subsection{Ablation on $\log L$ Regularization (RQ4)}
Fig.~\ref{fig:ablation_logL_bar} ablates the $\log L$ regularization term in \Framework{} on HumanEval-Infilling.
Introducing $\log L$ regularization yields consistent gains across both backbones, and the improvements are notably larger in the Multi-line setting, which is more sensitive to length ambiguity.
In contrast, Single-line shows smaller but still positive changes, suggesting that when the target span is short and well constrained, length selection is less of a bottleneck.

Qualitatively, these trends align with the role of the $\log L$ term as an explicit complexity penalty during inference: it discourages unnecessarily long candidate spans, thereby reducing redundant continuation and helping the model commit to a more coherent span that better matches the prefix--suffix context.
This effect becomes most pronounced for Multi-line infilling, where the search space over plausible lengths is widest and naive confidence-based selection is most prone to over-generation.

\subsection{Ablation on Stage~II (RQ5)}
Fig.~\ref{fig:ablation_stage2_bar} ablates Stage~II (mid-span adjustment) in \Framework{} on HumanEval-Infilling.
Across both backbones, enabling Stage~II yields substantial and consistent gains, with the improvements most pronounced on Multi-line, where failures often arise from local inconsistencies inside long spans (e.g., drifting variables, broken control flow, or mismatched intermediate states) rather than only boundary length mismatch.
On Dream-7B, Stage~II increases Single-line Pass@1 from 63.8 to 76.7 (+12.9; +20.2\% relative) and boosts Multi-line from 32.3 to 46.7 (+14.4; +44.6\% relative).
A similar pattern holds for DreamCoder-7B: Single-line rises from 68.3 to 81.6 (+13.3; +19.5\% relative), and Multi-line increases from 35.6 to 53.8 (+18.2; +51.1\% relative).

These trends suggest that Stage~II contributes more than coarse length correction: it actively refines the span interior, fixing mid-span errors that would otherwise persist even when the overall length is plausible.
As a result, \Framework{} achieves markedly better Multi-line coherence, indicating that mid-span refinement is a crucial complement to Stage~I's boundary-driven length adjustment.

\section{Related Work}
\label{sec:related}

\subsection{Variable-Length Generation in Diffusion Large Language Models (DLLMs)}
Unlike autoregressive models~\cite{guo2025deepseek, yang2025qwen3, Bee}, most DLLMs assume a \emph{fixed}-length canvas, which conflicts with variable-length FIM. Current methods are either \textit{training-based} or \textit{training-free}.

\begin{figure}[t]
    \centering
    \includegraphics[width=\linewidth]{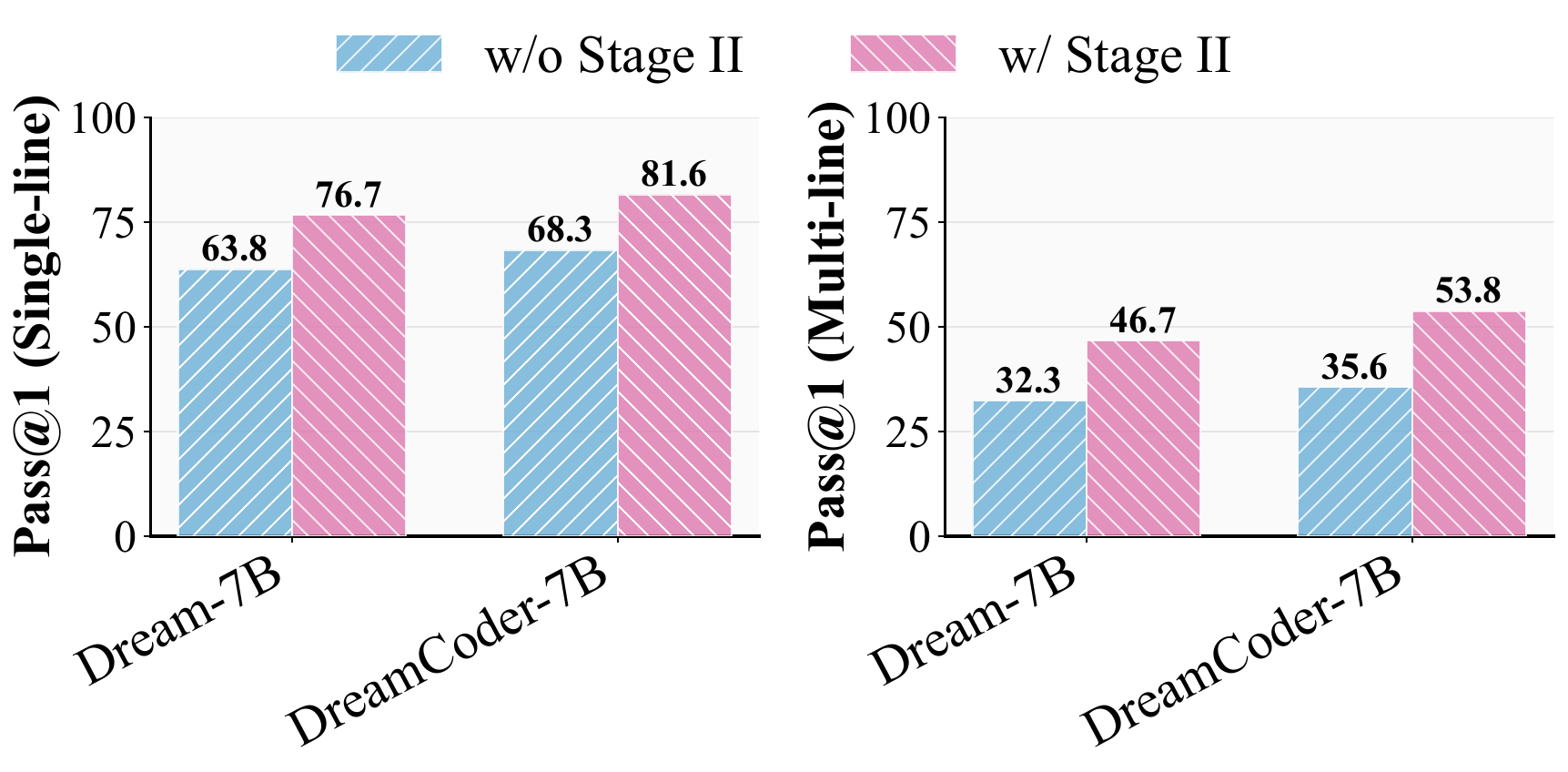}
    \caption{Effect of Stage~II (mid-span adjustment) in \Framework{} on HumanEval-Infilling (Pass@1) for Dream-7B and DreamCoder-7B.}
    \label{fig:ablation_stage2_bar}
\end{figure}

\paragraph{Training-based methods.}
DreamOn trains special edit tokens (e.g., \texttt{<expand>}, \texttt{<delete>}) and applies custom inference rules to
insert or remove masked positions~\cite{Dreamon2025}.
AnyOrder redesigns denoising to jointly predict tokens and inter-token mask counts for flexible length~\cite{FlexMDM}.
Flexible changes token positions during generation to adjust the number of masks, also requiring training-time
modifications~\cite{Flexible}.
dLLM-var is another training-based approach that relies on explicit termination signals (e.g., EOS-style stopping) for
variable-length generation~\cite{dLLM-Var}, which is less convenient for arbitrary-span FIM or settings without a clear
EOS boundary.
Overall, these methods improve flexibility but incur extra training cost and may reduce compatibility with strong
off-the-shelf DLLMs.
In contrast, \Framework{} is training-free and preserves the original DLLM denoising rule while enabling stable
bidirectional length adjustment.

\paragraph{Training-free methods.}
Beyond-Fixed studies training-free variable-length generation via inference-time heuristics for selecting or adjusting
the length canvas~\cite{DAEDAL}. While avoiding retraining, such heuristics can be sensitive to initialization and brittle
under large length variance in FIM.
\Framework{} is also training-free, but treats length as an explicit inference variable and supports both expansion and
contraction, reducing sensitivity to the initial mask length.

\subsection{Diffusion Large Language Models (DLLMs)}
DLLMs conduct iterative denoising with bidirectional context and have recently shown strong performance on general and code generation.
Representative large-scale DLLMs include the LLADA family~\cite{LLADA,LLADA-1.5,LLADA-MOE}, DEER~\cite{cheng2025deer} and Dream~\cite{Dream}. However, most DLLM inference pipelines remain fixed-length, which restricts their applicability to realistic FIM and completion where output length is not given.
Our work complements these DLLMs by enabling variable-length inference without modifying or retraining the underlying model, making existing DLLMs directly usable in unknown-length constrained generation.

\section{Conclusion}
We presented \Framework{}, a training-free framework for unknown-length constrained generation in diffusion LLMs.
By treating length as an explicit inference variable and optimizing a debiased length--confidence signal, \Framework{}
supports bidirectional length adjustment (expansion and contraction) without modifying the base model.
Beyond infilling, we evaluate \Framework{} on general-purpose benchmarks (MATH500, GSM8K, and HumanEval), demonstrating
that it also improves standard left-to-right generation.
\newpage
\section{Impact Statement}
This paper presents work whose goal is to advance the field of machine learning. There are many potential societal consequences of our work, none of which we feel must be specifically highlighted here.

\bibliography{example_paper}
\bibliographystyle{icml2026}

\appendix
\onecolumn
\section{Experimental Setup and Details}

\paragraph{Baseline availability and coverage.}
We compare against representative variable-length diffusion baselines when their public releases allow
reproducible evaluation under our settings.
However, some contemporary methods cannot be uniformly evaluated across backbones due to limitations in
released checkpoints and implementation availability.
\textbf{DreamOn}~\cite{Dreamon2025} is only publicly released with \texttt{DreamCoder} checkpoints, and
training the full system from scratch is computationally expensive (e.g., requiring large-scale multi-GPU
training as reported by the authors); therefore, we report DreamOn comparisons only on the DreamCoder backbone.
\textbf{FlexMDM}~\cite{FlexMDM} does not provide public code at the time of writing; we thus cite the
numbers reported in the original paper, noting that the paper reports results only for a limited subset of
settings.
Finally, \textbf{DAEDAL}~\cite{DAEDAL} is designed around the \texttt{LLaDA} family and is not directly
applicable to \texttt{Dream} under its current framework; hence, we include DAEDAL only in the settings
where it is compatible.
Across all reported comparisons, we align inference budgets and evaluation protocols as closely as possible
to ensure fairness, and we explicitly mark settings that are not applicable or not available.

\paragraph{Decoding and inference configuration.}
Unless otherwise specified, all experiments use the same decoding configuration for fair comparison.
Specifically, we apply standard LLM sampling with nucleus sampling probability $p{=}0.9$,
no top-$k$ truncation, and temperature $0.2$ across all benchmarks.
The same decoding settings are used for \Framework{} and all baseline methods whenever applicable. All reported results are obtained under fixed random seeds and deterministic decoding settings.
We observed negligible variance across repeated runs (<0.5 Pass@1 on HumanEval-Infilling).
\begin{table}[t]
\centering
\small
\setlength{\tabcolsep}{8pt}
\renewcommand{\arraystretch}{1.15}
\caption{
Unified decoding configuration used across benchmarks.
All results for \Framework{} and baselines are obtained under the same settings unless otherwise noted.
}
\label{tab:decoding_config}
\begin{tabular}{lccc}
\toprule
\textbf{Benchmark} & \textbf{Top-$p$} & \textbf{Top-$k$} & \textbf{Temperature} \\
\midrule
HumanEval-Infilling & 0.9 & None & 0.2 \\
McEval              & 0.9 & None & 0.2 \\
MATH500             & 0.9 & None & 0.2 \\
GSM8K               & 0.9 & None & 0.2 \\
HumanEval           & 0.9 & None & 0.2 \\
\bottomrule
\end{tabular}
\end{table}

\section{Run-to-Run Variability Analysis}
Across all models and infilling settings, the run-to-run variability of LR-DLLM remains consistently small. 
As shown in Table~\ref{tab:he_var_lrdllm}, the standard deviation of Pass@1 is below $0.5$ in all cases, 
indicating that the observed performance gains are not driven by stochastic fluctuations across runs. 
This stability holds across heterogeneous backbones and span regimes, suggesting that the proposed 
length-regularized inference yields reliable length decisions rather than exhibiting brittle or oscillatory behavior.

In particular, the magnitude of variability is an order of magnitude smaller than the absolute improvements 
reported in the main results, confirming that the benefits of LR-DLLM are stable rather than within noise.

\begin{table}[t]
\centering
\small
\setlength{\tabcolsep}{6pt}
\renewcommand{\arraystretch}{1.1}
\begin{tabular}{lccc}
\toprule
\textbf{Model} & \textbf{Random Span} ($\pm$ std) & \textbf{Single-line} ($\pm$ std) & \textbf{Multi-line} ($\pm$ std) \\
\midrule
LLaDA-8B        & $19.2\pm0.1$ & $69.4\pm0.3$ & $35.9\pm0.2$ \\
LLaDA-1.5       & $20.5\pm0.2$ &$68.9\pm0.2$ & $34.6\pm0.4$ \\
LLaDA-MoE       & $21.5\pm0.1$ & $71.3\pm0.3$ & $41.4\pm0.4$ \\
Dream-7B        & $19.6\pm0.2$ & $76.7\pm0.4$ & $46.3\pm0.2$ \\
DreamCoder-7B   & $23.2\pm0.2$ & $81.6\pm0.3$ & $53.8\pm0.4$ \\
\bottomrule
\end{tabular}
\caption{\textbf{Run-to-run variability of LR-DLLM on HumanEval-Infilling.}
We report the standard deviation of Pass@1 across $N$ independent runs (different random seeds) for each setting.}
\label{tab:he_var_lrdllm}
\end{table}

\section{Clarification on Task-Agnosticity}
\label{app:task-ag}
In this paper, the term \emph{task-agnostic} refers to the \emph{structural generality} of the proposed
inference framework, rather than empirical validation across all possible task categories.
Specifically, \Framework{} operates on a generic prefix--suffix composition
$\mathbf{x}=\mathbf{p}\Vert\mathbf{y}\Vert\mathbf{s}$ and does not rely on task-specific supervision,
loss functions, or architectural modifications.

In practice, this formulation subsumes two broad and commonly used generation paradigms:
(i) \emph{prefix--suffix infilling} (FIM), where both $\mathbf{p}$ and $\mathbf{s}$ are non-empty, and
(ii) \emph{left-to-right continuation}, which is recovered as the special case $\mathbf{s}=\varnothing$.
Our empirical evaluation focuses on representative tasks within these paradigms, including code
completion, multi-line infilling, and mathematical reasoning, which naturally instantiate the above
structure.

We do not claim that \Framework{} has been empirically validated on all task families
(e.g., dialogue, summarization, or instruction following).
Rather, the task-agnostic claim is intended in the sense that the method is \emph{not specialized to any
particular task}, and can in principle be applied to any generation problem that can be cast into a
prefix--suffix composition with unknown span length.

\section{Fixed-Length Mask Sweep on HumanEval-Infilling}
\label{app:fixed_length_sweep}

\paragraph{Setup.}
We evaluate fixed-length infilling on HumanEval-Infilling using the DreamCoder backbone by sweeping
the mask length $L \in \{4, 8, 16, 32, 64\}$ under identical inference budgets and decoding settings.
Pass@1 accuracy is reported separately for Single-Line and Multi-Line infilling.
Table~\ref{tab:fixed_length_sweep_humaneval} summarizes the numerical results, and
Figure~\ref{fig:humaneval_length_sweep} visualizes the performance trends.

\paragraph{Results.}
As shown in Table~\ref{tab:fixed_length_sweep_humaneval}, fixed-length performance exhibits a strong
dependence on the chosen mask length.
For Single-Line infilling, accuracy peaks at $72.6\%$ with $L{=}16$ and degrades for larger masks
(e.g., $62.4\%$ at $L{=}32$ and $55.5\%$ at $L{=}64$).
For Multi-Line infilling, performance improves with increasing $L$ up to $43.2\%$ at $L{=}64$ but
remains substantially below \Framework{} ($53.8\%$).
Across both settings, \Framework{} consistently outperforms the best fixed-length choice selected
in hindsight (Table~\ref{tab:fixed_length_sweep_humaneval}), indicating that its gains cannot be
attributed solely to increasing $\text{MAX\_LENGTH}$.

While ground-truth span lengths are available for HumanEval-Infilling, we focus on downstream correctness as the primary criterion, since multiple lengths may admit semantically valid solutions. Direct length error metrics (e.g., MAE) are therefore not fully aligned with task success.

\label{app:max_length_tradeoff}
\begin{table}[t]
\centering
\small
\setlength{\tabcolsep}{6pt}
\renewcommand{\arraystretch}{1.15}
\caption{
Comparison between fixed-length infilling with different mask sizes and \Framework{}
on HumanEval-Infilling using the DreamCoder backbone.
We report Pass@1 accuracy (\%).
}
\label{tab:fixed_length_sweep_humaneval}
\begin{tabular}{lcccccc}
\toprule
\textbf{Benchmark} &
\multicolumn{5}{c}{\textbf{Fixed Mask Length}} &
\multirow{2}{*}{\textbf{\Framework{}}} \\
\cmidrule(lr){2-6}
 & 4 & 8 & 16 & 32 & 64 &  \\
\midrule
Single-Line & 24.9 &61.2& 72.6& 62.4& 55.5 & \ours{\textbf{81.6}} \\
Multi-Line  & 5.5& 14.7& 27.1& 39.4& 43.2 & \ours{\textbf{53.8}} \\
\bottomrule
\end{tabular}
\end{table}

\begin{figure}[t]
    \centering
    \begin{subfigure}{0.48\linewidth}
        \centering
        \includegraphics[width=\linewidth]{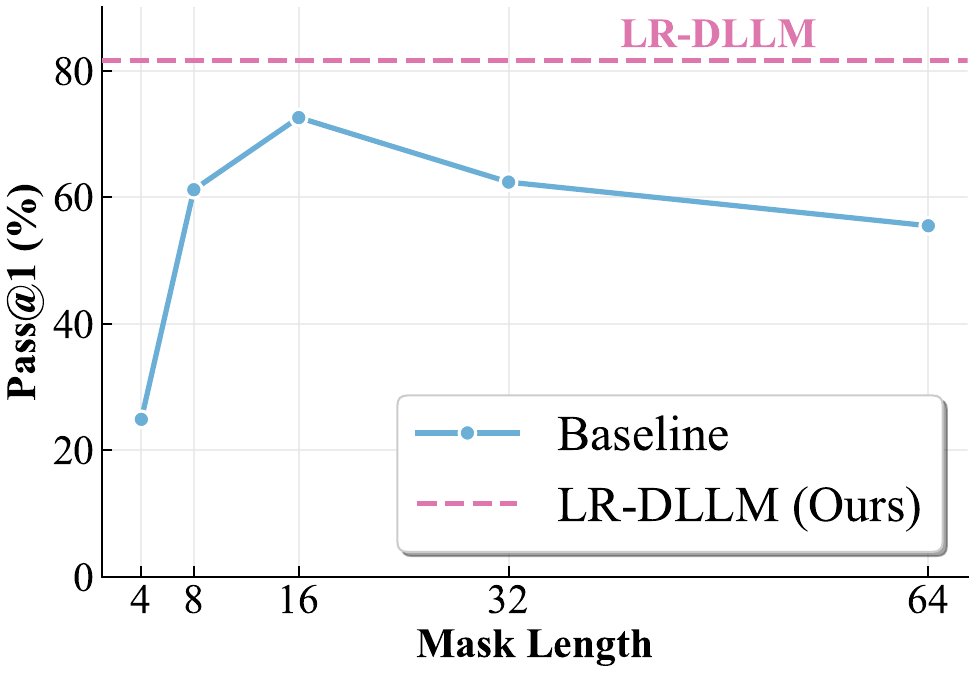}
        \caption{Single-Line Infilling}
        \label{fig:humaneval_single}
    \end{subfigure}
    \hfill
    \begin{subfigure}{0.48\linewidth}
        \centering
        \includegraphics[width=\linewidth]{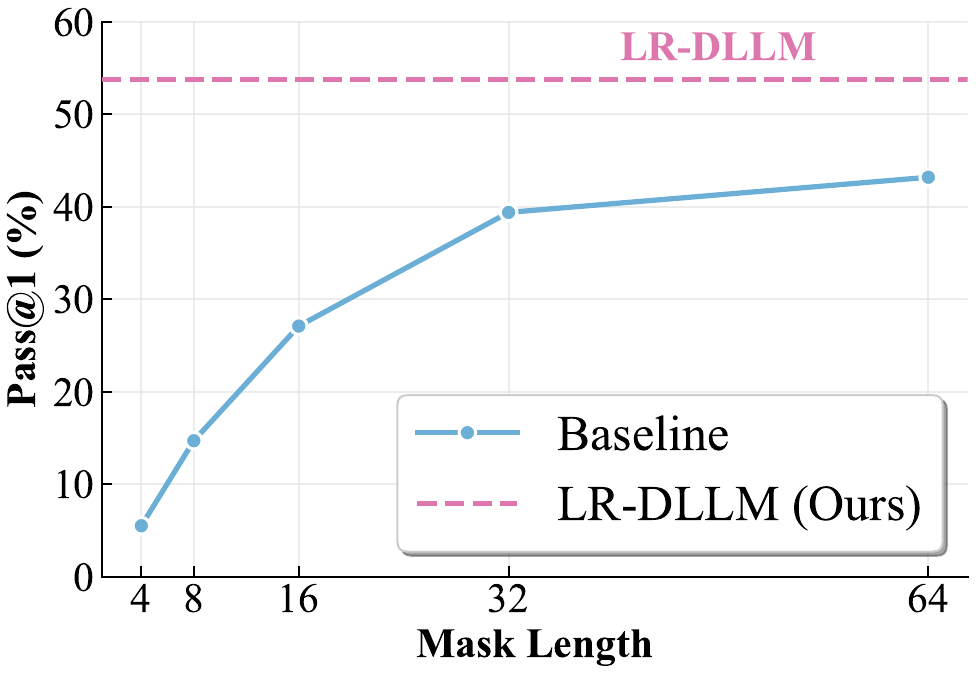}
        \caption{Multi-Line Infilling}
        \label{fig:humaneval_multi}
    \end{subfigure}

    \caption{
    Effect of fixed mask length on HumanEval-Infilling.
    Solid lines denote fixed-length infilling with different mask sizes,
    while the dashed horizontal line corresponds to \Framework{}.
    }
    \label{fig:humaneval_length_sweep}
\end{figure}

\section{Length-Induced Entropy in FIM and Continuation}
\label{app:logL_bias}
In both fill-in-the-middle (FIM) and standard left-to-right continuation, we consistently observe a systematic trend:
as the ``free'' generation region becomes longer, token predictions become less certain, and uncertainty tends to increase
as tokens move farther away from the anchoring context (i.e., the boundary between conditioned and unconditioned spans).
This phenomenon suggests that measured uncertainty contains an intrinsic length-dependent component in addition to the
model's quality-dependent component.

\paragraph{A statistical-physics view.}
We model this effect using the classical relation between entropy and the number of admissible microstates:
\begin{equation}
S = k \log N,
\label{eq:boltzmann}
\end{equation}
where $N$ denotes the number of feasible configurations under macroscopic constraints and $k$ is a constant scale factor.
In language generation, macroscopic constraints are induced by the given context (prefix/suffix, syntax, semantics), while
microstates correspond to distinct valid token (or sequence) realizations consistent with those constraints.

\paragraph{Assumption: admissible choices increase with distance from the anchor.}
Consider a contiguous free region of length $L$.
Index positions by their distance $d \in \{1,2,\dots,L\}$ from an anchoring point.\footnote{For FIM, one may interpret $d$
as the distance to the nearest boundary (prefix or suffix). The analysis below focuses on the dependence on $L$ and thus
applies qualitatively to both settings.}
We assume the effective number of admissible choices at distance $d$ grows as a power law:
\begin{equation}
\#\text{choices}(d) \propto d^{r}, \qquad r>0,
\label{eq:choices}
\end{equation}
capturing the intuition that contextual constraints weaken with distance.

\paragraph{Total number of configurations.}
Using a multiplicative approximation across positions, the number of feasible configurations for a length-$L$ free region
scales as
\begin{equation}
N(L) \propto \prod_{d=1}^{L} d^{r} = (L!)^{r}.
\label{eq:NL}
\end{equation}

\paragraph{Length-induced entropy and its $\log L$ growth.}
Substituting Eq.~\eqref{eq:NL} into Eq.~\eqref{eq:boltzmann}, the entropy induced purely by the growth of admissible
configurations with $L$ is
\begin{equation}
S_{2}(L) = k \log N(L) \approx k \log (L!)^{r} = kr \log(L!).
\label{eq:S2_exact}
\end{equation}
Applying Stirling's approximation,
\begin{equation}
\log(L!) = L\log L - L + \mathcal{O}(\log L),
\label{eq:stirling}
\end{equation}
we obtain the leading-order behavior
\begin{equation}
S_{2}(L) \approx kr\,(L\log L - L) \sim kr\,L\log L.
\label{eq:S2_asymp}
\end{equation}
Define the average (per-token) length-induced entropy:
\begin{equation}
\overline{S_{2}}(L) := \frac{S_{2}(L)}{L}
\approx kr\log L - kr + \mathcal{O}\!\left(\frac{\log L}{L}\right).
\label{eq:S2_avg}
\end{equation}
Hence, up to an additive constant,
\begin{equation}
\overline{S_{2}}(L) \approx a \log L + b,
\qquad \text{with } a \approx kr.
\label{eq:S2_fit}
\end{equation}
This predicts that even with a fixed model, average predictive uncertainty should increase approximately linearly in
$\log L$ as the free region grows.

\paragraph{Decomposition: length effect vs.\ model quality.}
Let $\mathrm{AVG\_S}(L)$ denote a measurable average uncertainty statistic (e.g., average token entropy computed from
model logits) at free length $L$. We decompose it as
\begin{equation}
\mathrm{AVG\_S}(L) = S_{1}(L) + \overline{S_{2}}(L),
\label{eq:decompose}
\end{equation}
where $\overline{S_{2}}(L)$ captures the inevitable length-induced contribution characterized above, and $S_{1}(L)$ is a
quality-related term reflecting the model's effective certainty under the current constraints.

\paragraph{Practical estimation via log-length fitting.}
Although $S_{1}(L)$ is not directly observable, $\mathrm{AVG\_S}(L)$ is measurable and $\overline{S_{2}}(L)$ admits the
log-length form in Eq.~\eqref{eq:S2_fit}. We therefore fit
\begin{equation}
\widehat{\overline{S_{2}}}(L) = a\log L + b,
\label{eq:fit_model}
\end{equation}
by regressing measured $\mathrm{AVG\_S}(L)$ over a small set of geometrically spaced lengths, e.g.,
\begin{equation}
L \in \{\mathrm{MAX},\, \mathrm{MAX}/2,\, \mathrm{MAX}/4,\, \dots\}.
\label{eq:geometric_lengths}
\end{equation}
We then estimate the quality-related component as
\begin{equation}
\widehat{S_{1}}(L) = \mathrm{AVG\_S}(L) - \widehat{\overline{S_{2}}}(L).
\label{eq:S1_est}
\end{equation}

\paragraph{Interpretation.}
The estimate $\widehat{S_{1}}(L)$ serves as a length-normalized uncertainty signal: it subtracts the predictable entropy
growth attributable to increased free length, making uncertainty more comparable across different candidate lengths in
both FIM and continuation settings.

\section{Forward-Pass Complexity of \Framework{}}
\label{app:appendix_complexity}

We analyze the complexity of \Framework{} in terms of the number of model forward evaluations.
We count one \emph{forward evaluation} as a single invocation of the DLLM inference procedure on a template
$\mathbf{x}^{(0)}(L)=\mathbf{p}\Vert \langle \mathrm{mask}\rangle^{L}\Vert \mathbf{s}$, which returns
$\mathrm{AVG}_{\mathrm{conf}}(L;\mathbf{p},\mathbf{s})$ (and the associated token distributions).

\paragraph{Stage~I (exponential probing).}
Stage~I evaluates $\mathrm{AVG}_{\mathrm{conf}}(L;\mathbf{p},\mathbf{s})$ on the exponentially spaced set
$\mathcal{L}_{\exp}=\{1,2,4,\dots,2^m\}$ where $2^m \le L_{\max}$.
Thus Stage~I uses
\begin{equation}
F_{\text{I}} = |\mathcal{L}_{\exp}| = \lfloor \log_2 L_{\max}\rfloor + 1 = O(\log L_{\max})
\end{equation}
forward evaluations.

\paragraph{Stage~II (greedy local search).}
At each commitment step $t\in\{1,\dots,N\}$, the algorithm performs greedy hill-climbing over lengths by comparing
$\mathrm{CL}$ on neighboring lengths $\{L-1,L,L+1\}$, and then commits one token once a strict local maximum is reached.
Let $u_t$ denote the number of length updates (moves) performed before the $t$-th commitment.

\begin{lemma}[Monotonicity of the length search]
\label{lem:mono_search}
Within a fixed commitment step (i.e., with $\mathbf{p},\mathbf{s}$ held fixed), the greedy update rule never revisits a
previous length. Hence, the length trajectory is monotone and terminates at a local maximum.
\end{lemma}

\begin{proof}[Proof sketch]
Suppose the algorithm moves from $L$ to $L+1$. By definition, $\mathrm{CL}(L+1)>\mathrm{CL}(L)$.
At the next iteration (still under the same context), $L$ is a neighbor of $L+1$, so moving back would require
$\mathrm{CL}(L)>\mathrm{CL}(L+1)$, contradicting the strict improvement. The same argument holds for moves to $L-1$.
\end{proof}

\paragraph{Counting forward evaluations in Stage~II.}
At the beginning of each commitment step, we evaluate $\mathrm{CL}$ at $\{L-1,L,L+1\}$, requiring at most $3$ forwards.
By Lemma~\ref{lem:mono_search}, each move changes $L$ by $\pm 1$ and introduces at most one new boundary length whose
$\mathrm{CL}$ has not been evaluated yet; therefore each additional move costs at most one forward evaluation.
Moreover, in our implementation the inference outputs at the locally optimal length are reused to commit the token, so
no extra forward is needed for commitment.\footnote{If an additional forward is used to commit, it only changes constants.}
Therefore, the total number of forward evaluations in Stage~II satisfies
\begin{equation}
F_{\text{II}} \;\le\; \sum_{t=1}^{N} (3 + u_t)
\;=\; 3N + U,
\qquad\text{where } U:=\sum_{t=1}^{N} u_t.
\label{eq:F2_bound}
\end{equation}

\paragraph{Worst-case upper bound.}
Since the length is always within $[1,L_{\max}]$, we have $u_t \le L_{\max}-1$ and hence $U \le N(L_{\max}-1)$.
Plugging into Eq.~\eqref{eq:F2_bound} yields
\begin{equation}
F_{\text{II}} = O(NL_{\max}).
\end{equation}
When $L_{\max}=O(N)$, this gives the conservative worst-case bound $F_{\text{II}}=O(N^2)$.

\paragraph{A tighter bound under a stability assumption.}
Let $L_t^\star$ denote the (strict) local maximizer reached at commitment step $t$ under the corresponding context.
Since we always initialize step $t$ from the previous locally optimal length after committing one token,
we have $L^{\text{init}}_{t} = L_{t-1}^\star - 1$ for $t\ge 2$.
By Lemma~\ref{lem:mono_search}, the number of moves is exactly the distance to the new local optimum:
\begin{equation}
u_t = \big|L_t^\star - L_t^{\text{init}}\big|
= \big|L_t^\star - (L_{t-1}^\star - 1)\big|
\le 1 + |L_t^\star - L_{t-1}^\star|.
\end{equation}
Summing over $t$ gives
\begin{equation}
U \le N + \sum_{t=2}^{N} |L_t^\star - L_{t-1}^\star|.
\label{eq:U_tv}
\end{equation}

\begin{assumption}[Bounded drift of locally optimal lengths]
\label{ass:drift}
The locally optimal length does not change abruptly across successive subproblems, i.e.,
$\sum_{t=2}^{N} |L_t^\star - L_{t-1}^\star| = O(NH)$ for some $H$.
\end{assumption}

Under Assumption~\ref{ass:drift}, Eq.~\eqref{eq:U_tv} implies $U=O(NH)$, and therefore Eq.~\eqref{eq:F2_bound} yields
\begin{equation}
F_{\text{II}} = O(NH).
\end{equation}
In particular, if the drift satisfies $H=O(\log L_{\max})$ (or if $|L_t^\star-L_{t-1}^\star| \le c\log L_{\max}$
uniformly), then
\begin{equation}
F_{\text{II}} = O(N\log L_{\max}).
\end{equation}

\paragraph{Overall bound.}
Combining Stage~I and Stage~II, the total number of forward evaluations is
\begin{equation}
F_{\text{total}} = F_{\text{I}} + F_{\text{II}}
= O(\log L_{\max}) + O(NL_{\max})
\end{equation}
in the worst case, and
\begin{equation}
F_{\text{total}} = O(\log L_{\max}) + O(NH)
\end{equation}
under Assumption~\ref{ass:drift}. In practice, we observe that $H$ is small, so the empirical complexity is close to linear.

\section{Forward-Call Efficiency Distribution Across Models}
\label{sec:forward_call_ratio}

We quantify the inference-time overhead of \Framework{} using the ratio
$\texttt{Forward Calls}/\texttt{Generated Tokens}$, where \texttt{Forward Calls} counts the number of model invocations
used to evaluate $\mathrm{AVG}_{\mathrm{conf}}(L)$ (and thus $\mathrm{CL}(L)$) during Stage~I probing and Stage~II local
length search, and \texttt{Generated Tokens} is the number of committed tokens produced by Stage~II.
Figure~\ref{fig:fc_ratio_three_models} visualizes the per-instance distribution of this ratio for three backbones
(\textbf{Dream}, \textbf{DreamCoder}, and \textbf{LLaDA-8B}) under an identical plotting protocol, while
Table~\ref{tab:fc_ratio_stats} reports summary statistics.

Across all three models, the distributions are markedly right-skewed: the median is consistently lower than the mean
(e.g., Dream: 3.43 vs.\ 3.66; DreamCoder: 4.00 vs.\ 4.89; LLaDA-8B: 4.63 vs.\ 5.12), indicating that most instances incur
a stable and moderate overhead, while a small number of hard cases forms a long tail and inflates the mean.
Using the median as a robust measure of typical behavior, Dream is the most efficient backbone (median 3.43), followed by
DreamCoder (median 4.00) and LLaDA-8B (median 4.63). The tails differ more substantially: Dream exhibits the tightest
distribution (P95 = 5.00), whereas DreamCoder shows the heaviest tail (P95 = 11.23), with LLaDA-8B in between (P95 = 8.00).
All three models have rare extreme outliers (max around 28--33), which are far beyond their respective P95 values, suggesting
that they occur infrequently but can arise when many local length updates are triggered for instances with short effective
generations or unstable local optima.

Overall, these results corroborate that, despite a conservative worst-case bound, \Framework{} typically converges in only a
few local updates per committed token on standard benchmarks, leading to a modest overhead for the vast majority of instances.

\begin{table}[t]
\centering
\small
\setlength{\tabcolsep}{6.0pt}
\renewcommand{\arraystretch}{1.15}
\caption{Summary statistics of the per-instance ratio $\texttt{Forward Calls}/\texttt{Generated Tokens}$ computed from
generation logs (lower is better). Here \texttt{Forward Calls} counts model invocations used to evaluate
$\mathrm{AVG}_{\mathrm{conf}}(L)$ (hence $\mathrm{CL}(L)$) during Stage~I and Stage~II, and \texttt{Generated Tokens} is the
number of committed tokens.}
\label{tab:fc_ratio_stats}
\begin{tabular}{lcccccc}
\toprule
\textbf{Model} & \textbf{n} & \textbf{Mean} & \textbf{Median} & \textbf{P95} & \textbf{Min} & \textbf{Max} \\
\midrule
Dream       & 1033 & 3.6574 & 3.4286 & 5.0000  & 2.0625 & 32.8333 \\
DreamCoder  & 1033 & 4.8910 & 4.0000 & 11.2268 & 2.8594 & 27.6667 \\
LLaDA-8B    & 1033 & 5.1211 & 4.6316 & 8.0000  & 2.7288 & 31.1429 \\
\bottomrule
\end{tabular}
\vspace{-4pt}
\end{table}

\begin{figure*}[t]
    \centering
    \begin{subfigure}[t]{0.32\textwidth}
        \centering
        \includegraphics[width=\linewidth]{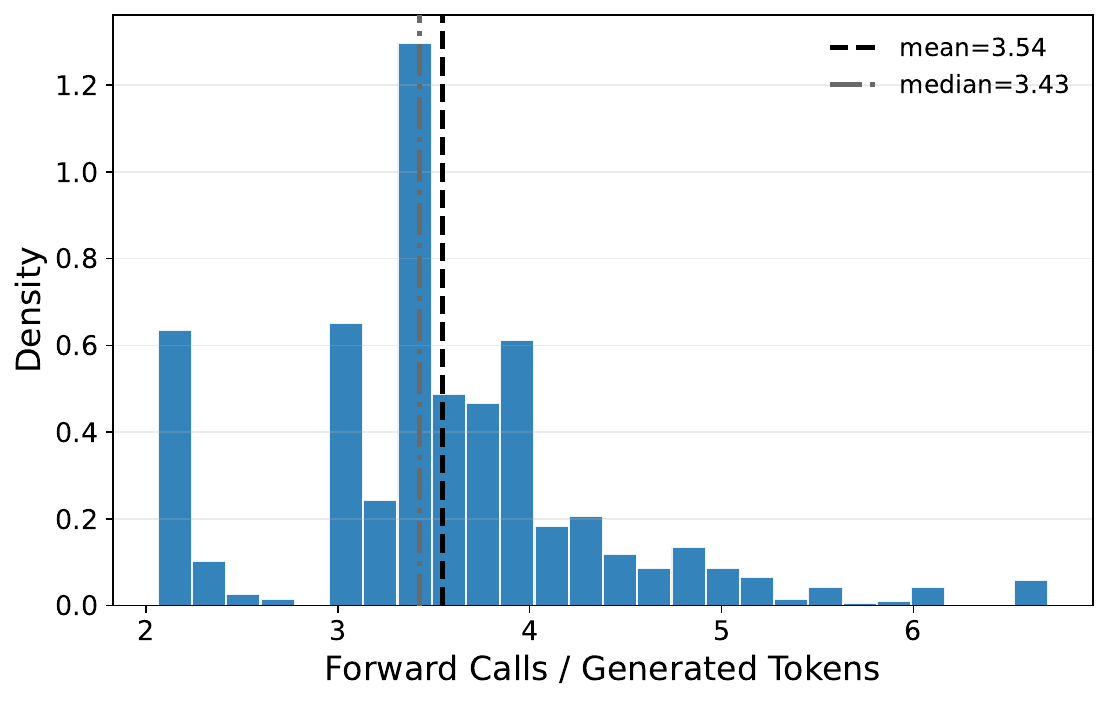}
        \caption{\textbf{Dream}.}
        \label{fig:fc_ratio_dream}
    \end{subfigure}
    \hfill
    \begin{subfigure}[t]{0.32\textwidth}
        \centering
        \includegraphics[width=\linewidth]{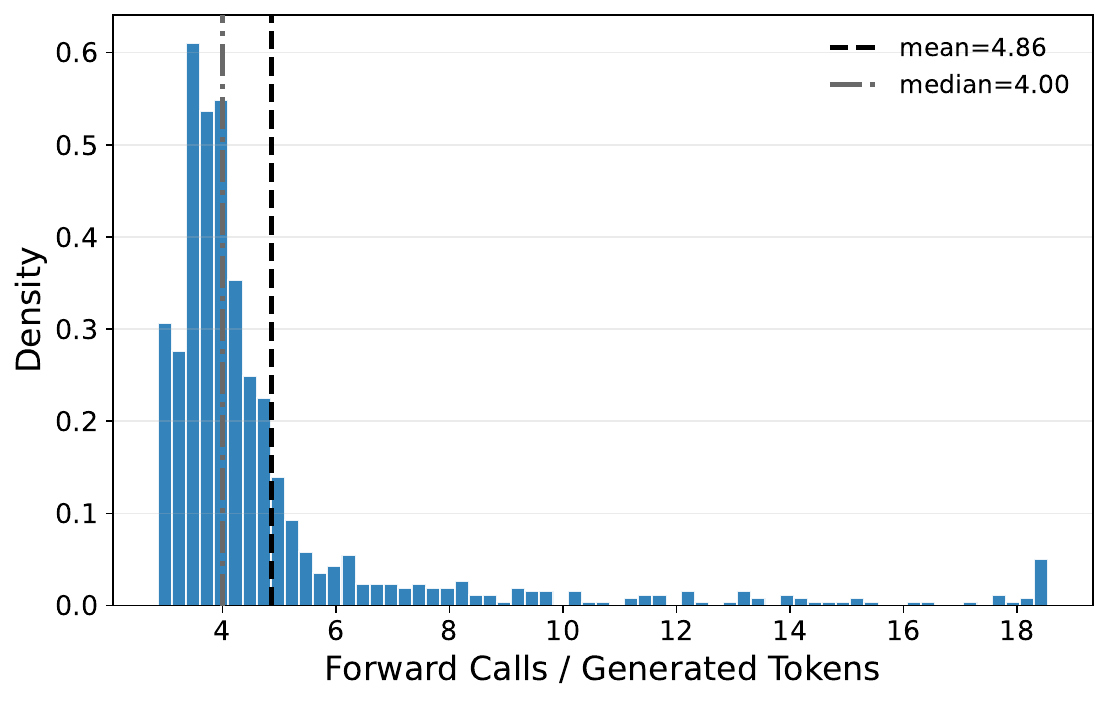}
        \caption{\textbf{DreamCoder}.}
        \label{fig:fc_ratio_dreamcoder}
    \end{subfigure}
    \hfill
    \begin{subfigure}[t]{0.32\textwidth}
        \centering
        \includegraphics[width=\linewidth]{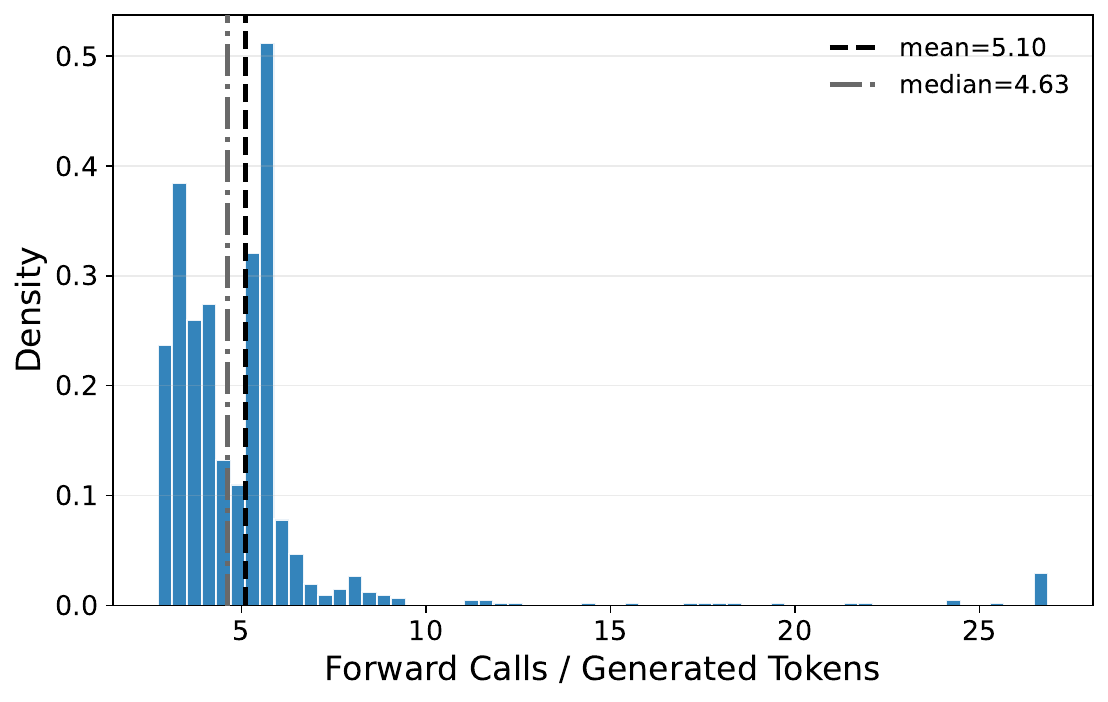}
        \caption{\textbf{LLaDA-8B}.}
        \label{fig:fc_ratio_llada8b}
    \end{subfigure}

    \vspace{-4pt}
    \caption{Per-instance distributions of the overhead ratio $\texttt{Forward Calls}/\texttt{Generated Tokens}$ for three
    backbones under the same inference and logging protocol. \texttt{Forward Calls} counts model invocations used to evaluate
    $\mathrm{AVG}_{\mathrm{conf}}(L)$ (and thus $\mathrm{CL}(L)$) during Stage~I probing and Stage~II greedy local length
    search, and \texttt{Generated Tokens} is the number of committed tokens produced by Stage~II. The right-skewed shapes
    indicate a stable typical overhead with a small number of hard outliers.}
    \label{fig:fc_ratio_three_models}
    \vspace{-8pt}
\end{figure*}

\section{Additional Comparison with dLLM-Var}
\label{app:dllmvar}

We include an auxiliary comparison with dLLM-Var, a representative \emph{training-based} variable-length method for
diffusion LLMs. We do \emph{not} report it in the main tables because its current public release is limited to a single
trained checkpoint (\texttt{LLADA-8B-Instruct}) and a specific inference variant (block diffusion), and it only supports
one-sided (backward-only) length growth. These constraints prevent a consistent evaluation across the multiple backbones
and general prefix--suffix FIM settings considered in this paper. Nevertheless, for completeness, we report results on
the subset of benchmarks that dLLM-Var can be evaluated on. For a fair comparison, we use the same maximum probing length
$\texttt{MAX\_LENGTH}=256$ for \Framework{} in all experiments below.

\begin{table}[t]
\centering
\small
\setlength{\tabcolsep}{5pt}
\renewcommand{\arraystretch}{1.15}
\caption{
Comparison with dLLM-Var on the \emph{restricted} subset of settings supported by its public release.
dLLM-Var is only available as a trained \texttt{LLADA-8B-Instruct} checkpoint with block-diffusion inference and
backward-only length growth; therefore it cannot be applied to other diffusion backbones or general prefix--suffix FIM
requiring bidirectional length adjustment. We report \Framework{} on the same backbone under training-free inference.
}
\label{tab:dllmvar}
\begin{tabular}{lcc}
\toprule
\textbf{Benchmark} & \textbf{dLLM-Var} & \textbf{\Framework{}} \\
\midrule
GSM8K (Acc.)          & 70.9  & 75.1 \\
MATH500 (Acc.)        & 18.4  & 39.8 \\
HumanEval (Pass@1)    & 17.7  & 31.7 \\
\bottomrule
\end{tabular}
\end{table}

\paragraph{Discussion.}
Table~\ref{tab:dllmvar} shows that \Framework{} consistently outperforms dLLM-Var on the supported subset.
Concretely, \Framework{} improves GSM8K accuracy from $70.9$ to $75.1$ (+4.2 points),
more than doubles performance on MATH500 from $18.4$ to $39.8$ (+21.4 points),
and boosts HumanEval Pass@1 from $17.7$ to $31.7$ (+14.0 points).
These gains are achieved while remaining training-free and model-agnostic, whereas dLLM-Var is tied to a specific trained
\texttt{LLADA-8B-Instruct} checkpoint and a block-diffusion procedure with one-sided length growth.
More importantly, \Framework{} supports \emph{bidirectional} length adjustment (both expansion and contraction) under
general constrained generation, which is essential for prefix--suffix FIM and settings with large length variance.
A broader open-source release of dLLM-Var covering more backbones and more general constrained generation settings would
enable a more comprehensive comparison in future work.

\section{Other Related Work}
\label{app:other_related_work}

Several recent works have explored diffusion-based language models (DLLMs) from complementary
perspectives.
\textbf{UltraLLaDA}~\cite{UltraLLaDA} focuses on extending the context window of DLLMs to very long
sequences (up to 128K tokens), addressing scalability in context length rather than the inference-time
determination of unknown span length.
\textbf{Diffusion-LM}~\cite{Diffusion-LM} and \textbf{S3}~\cite{S3} demonstrate that diffusion-based
formulations can be advantageous for conditional generation, but they operate under fixed-length
generation assumptions and do not address variable-length inference.

Other works investigate efficiency and decoding acceleration for DLLMs.
\textbf{DPad}~\cite{DPad} and \textbf{Sparse-dLLM}~\cite{Sparse-dLLM} study architectural and systems-level
optimizations, such as attention sparsification and KV-cache management, to improve inference speed.
These methods are orthogonal to our focus, as they assume a predetermined generation length and do not
consider length selection as part of the inference process.

Finally, \textbf{Theoretical Benefit and Limitation of Diffusion Language Models}~\cite{Limitation}
analyzes fundamental challenges of DLLMs, explicitly noting the difficulty of variable-length
generation.
Our work is complementary in that it directly targets this limitation by introducing an inference-time
framework for adaptive length determination, without modifying the underlying model or training
procedure.

\section{Case Study}

\subsection{Infilling}
\tcbset{colframe = prompt, colback = prompt!25!white,breakable}

\begin{tcolorbox}[title = {Python Infilling Example}] 
\begin{lstlisting}
def string_transformation(A: str, B: str, rules: list) -> Union[int, str]:

    queue = deque([(A, 0)])  # Each element in the queue is a tuple (current_string, steps_taken)
    visited = set([A])  # Keep track of visited strings to avoid loops

    # Perform BFS
    while queue:
        current, steps = queue.popleft()
        if current == B:
            return steps
        #BEGIN INFILLING
        for source, target in rules:
            new_current = current.replace(source, target)
            if new_current not in visited:
                visited.add(new_current)
                queue.append((new_current, steps + 1))
        #END INFILLING
    return "NO ANSWER!"
\end{lstlisting}
\end{tcolorbox}

\begin{tcolorbox}[title = {C++ Infilling Example}] 
\begin{lstlisting}
int countPermutations(int n, int k, int qq[])
{
    const int N = 505, P = 998244353;
    int q[N], m, dp[N][N], jc[N], f[N], ans;

    memset(q, 0, sizeof(q));
    memset(dp, 0, sizeof(dp));
    memset(jc, 0, sizeof(jc));
    memset(f, 0, sizeof(f));
    ans = 0;

    for (int i = 1; i <= n; i++)
        q[i] = qq[i - 1];

    dp[0][0] = f[0] = 1;
    ## START INFILLING
    for (int i = 1; i <= n; i++)
    {
        f[i] = f[i - 1];
    }

    for (int i = 1; i <= n; i++)
    {
        for (int j = 0; j < i; j++)
            for (int k = 1; k <= n; k++)
                dp[i][k] = (dp[i][k] + dp[j][k - 1] * 1LL * f[i - j] % P) % P;
    }
    ## END INFILLING
    m = 0;
    for (int i = 1; i <= n; i++)
        if (q[i] > q[i + 1])
        {
            m = i;
            break;
        }

    if (m == n)
    {
        for (int i = k; i <= n; i++)
            ans = (ans + dp[n][i]) % P;
    }
    else
    {
        for (int i = m + 1; i <= n; i++)
        {
            if (i != m + 1 && (q[i - 1] > q[i] || q[i] < q[m]))
                break;

            int c = k + i - n - 1;
            if (c >= 0)
                ans = (ans + dp[m][c] * 1LL * jc[i - m - 1] % P) % P;
        }
    }

    return ans;
}

\end{lstlisting}

\end{tcolorbox}

\subsection{Completion}

\begin{tcolorbox}[title = {Math Completion Example}] 
Question:

What is the least positive integer multiple of 30 that can be written with only the digits 0 and 2?

Answer:

To find the least positive integer multiple of 30 that can be written with only the digits 0 and 2, we need to follow these steps:

1. **Identify the properties of multiples of 30:**
- A number is a multiple of 30 if it is divisible by both 3 and 10.
- A number is divisible by 3 if the sum of its digits is divisible by 3.
- A number is divisible by 10 if its last digit is 0.

2. **Form the smallest number using the digits 0 and 2:**
- The smallest number we can form using the digits 0 and 2 is 20.

3. **Check if 20 is a multiple of 30:**
- The sum of the digits of 20 is \\(2 + 0 = 2\\), which is not divisible by 3.
- The last digit of 20 is 0, so it is divisible by 10.

4. **Form the next smallest number:**
- The next smallest number we can form using the digits 0 and 2 is 220.

5. **Check if 220 is a multiple of 30:**
- The sum of the digits of 220 is \\(2 + 2 + 0 = 4\\), which is not divisible by 3.
- The last digit of 220 is 0, so it is divisible by 10.

6. **Form the next smallest number:**
- The next smallest number we can form using the digits 0 and 2 is 2220.

7. **Check if 2220 is a multiple of 30:**
- The sum of the digits of 2220 is \\(2 + 2 + 2 + 0 = 6\\), which is divisible by 3.
- The last digit of 2220 is 0, so it is divisible by 10.

Since 2220 satisfies both conditions, it is the least positive integer multiple of 30 that can be written with only the digits 0 and 2.

Therefore, the answer is:
\\\ boxed\{2220\}
\end{tcolorbox}

\begin{tcolorbox}[title = {Code Completion Example}] 

\begin{lstlisting}
from typing import List\n\n\ndef rescale_to_unit(numbers: List[float]) -> List[float]:
\"\"\" Given list of numbers (of at least two elements), apply a linear transform to that list,
such that the smallest number will become 0 and the largest will become 1\n    >>> rescale_to_unit([1.0, 2.0, 3.0, 4.0, 5.0])
[0.0, 0.25, 0.5, 0.75, 1.0]
\"\"\"
    #Completion Start
    if len(numbers) < 2:
        raise ValueError(\"The list must contain at least two elements\")
    min_val = min(numbers)
    max_val = max(numbers)
    return [(x - min_val) / (max_val - min_val) for x in numbers]
\end{lstlisting}

\end{tcolorbox}

\end{document}